\documentclass{article}

\usepackage[table]{xcolor}
\usepackage{graphicx}
\usepackage{subcaption}
\usepackage[colorlinks, urlcolor=myblue, linkcolor=myblue, citecolor=gray]{hyperref}
\usepackage{fontawesome5} % in preamble
\usepackage{wrapfig}
\usepackage[main, final]{main}
\usepackage{multirow}

\usepackage[utf8]{inputenc} % allow utf-8 input
\usepackage[T1]{fontenc}    % use 8-bit T1 fonts
\usepackage{hyperref}       % hyperlinks
\usepackage{url}            % simple URL typesetting
\usepackage{booktabs}       % professional-quality tables
\usepackage{amsfonts}       % blackboard math symbols
\usepackage{amsmath}
\usepackage{nicefrac}       % compact symbols for 1/2, etc.
\usepackage{microtype}      % microtypography
\usepackage{xcolor}         % colors
\usepackage{amsthm}
\theoremstyle{plain}

\newtheorem{proposition}{Proposition}

\theoremstyle{definition}

\newtheorem{remark}{Remark}
\usepackage{thm-restate}

\usepackage[most]{tcolorbox}

\usepackage{amssymb}% http://ctan.org/pkg/amssymb
\usepackage{pifont}% http://ctan.org/pkg/pifont
\newcommand{\cmark}{\ding{51}}%
\newcommand{\xmark}{\ding{55}}%
\newcommand{\meanstd}[2]{#1\,{\scriptsize$\pm$\,#2}}
\newcommand{\claudelogo}{%
  \raisebox{-0.15ex}{\includegraphics[height=0.95em]{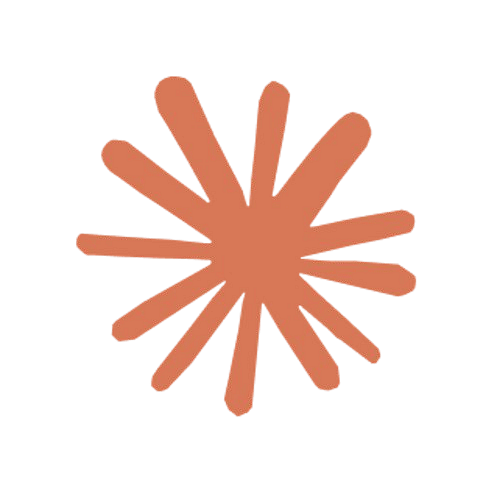}}%
}

\definecolor{mydarkred}{rgb}{0.6,0,0}
\definecolor{myblue}{HTML}{268BD2}
\definecolor{mygray}{HTML}{e9ecef}
\definecolor{lightpurple}{HTML}{F6E6FF}
\definecolor{color_bg}{HTML}{f5ebe0}
\definecolor{color_tt}{HTML}{c9ada7} 
\definecolor{color_fr}{HTML}{9a8c98} 
\definecolor{textblack}{HTML}{000000} 

\tcbset{
  mybox/.style={
    colback=color_bg,     
    colframe=color_fr,     
    coltitle=textblack,         
    fonttitle=\bfseries,      
    colbacktitle=color_tt,   
    enhanced,
    attach boxed title to top left={yshift=-1mm,xshift=2mm}, 
    boxed title style={
      colframe=color_fr,
      colback=color_tt,
      boxrule=1.2pt,
    },
    boxrule=1.2pt,            
    arc=2mm,                  
    top=4pt, bottom=4pt, left=6pt, right=6pt,
  }
}

\definecolor{lightbeige}{RGB}{245,240,230}

\newtcblisting{promptbox}{
  listing only,
  breakable,
  colback=lightbeige,
  colframe=black!20,
  boxrule=0.3pt,
  arc=1pt,
  left=2mm,
  right=2mm,
  top=1mm,
  bottom=1mm,
}

\title{Neglected Free Lunch from Post-training:\\Progress Advantage for LLM Agents}

\author{%
  Changdae Oh\textsuperscript{1}\quad
  Wendi Li\textsuperscript{1}\quad
  Seongheon Park\textsuperscript{1}\quad
  Samuel Yeh\textsuperscript{1}\\[0.2em]
  \textbf{Tanwi Mallick\textsuperscript{\textbf{2}}}\quad
  \textbf{Sharon Li\textsuperscript{\textbf{1}}}\\[0.2em]
  \textsuperscript{1}University of Wisconsin--Madison \quad
  \textsuperscript{2}Argonne National Laboratory \\
  \texttt{\{changdae,sharonli\}@cs.wisc.edu}
}

\begin{document}
\addtocontents{toc}{\protect\setcounter{tocdepth}{-1}}

\maketitle
\vspace{-1em}
\begin{abstract}
 Process reward models enable fine-grained, step-level evaluation of LLMs, yet building them for agentic settings remains prohibitively difficult: long-horizon interactions, irreversible actions, and stochastic environment feedback make both human annotation and Monte Carlo estimation infeasible at scale. 
 In this work, we show that reinforcement learning (RL) post-training already provides the ingredients for effective step-level scoring, eliminating the need for dedicated reward model training altogether. Concretely, we derive an implicit advantage under a general stochastic Markov decision process, which we term \emph{progress advantage}---log-probability ratio between the RL-trained policy and its reference policy exactly recovers the optimal advantage function.
 This formulation makes the resulting signal annotation-free, domain-agnostic, and available as a byproduct of the standard RL post-training pipeline. 
 We validate the effectiveness of the progress advantage across three different applications: \textit{test-time scaling}, \textit{uncertainty quantification}, and \textit{failure attribution} on five benchmarks and four model families. Across all settings, it consistently outperforms confidence-based baselines and, despite requiring no task-specific training, surpasses dedicated trained reward models. We complement these results with deeper analyses on characteristics of progress advantage, offering practical guidance for adoption in real-world agentic systems. URLs: \href{https://github.com/deeplearning-wisc/progress-advantage}{\faGithub}, \href{https://changdaeoh.github.io/progress-advantage/}{\faGlobe}
\end{abstract}

\section{Introduction} \label{sec:intro}

Reinforcement learning (RL) has become the dominant paradigm for post-training large language models (LLMs), producing agents that can operate autonomously across complex, multi-turn tasks involving tool use, web navigation, and code execution~\cite{silver2025welcome,openai2025agent,google2025agent,anthropic2026cowork}. A central challenge in deploying these agents is evaluating the quality of their behavior, so-called \textit{reward}, at the level of individual steps rather than only at the end of a trajectory. Outcome reward models assign a single scalar to a generated output~\cite{cobbe2021training,yu2024ovm,uesato2022solving,shao2024deepseekmath,guo2025deepseek}, but this coarse signal provides little guidance for credit assignment over trajectories that may span hundreds of actions. Process reward models (PRMs) address this by providing step-level supervision~\cite{lightman2023let,wang2024math,snell2025scaling, luo2024improve, li2025process}, enabling finer-grained trajectory evaluation that benefits test-time scaling, runtime monitoring, and failure diagnosis. However, PRMs have been explored mostly in mathematical reasoning, and they remain largely underexplored for LLM agents.

Unfortunately, building PRMs for LLM agents is notoriously difficult: agentic trajectories span long horizons and pass through stateful environments where actions such as sending an email or deleting a file are irreversible, preventing the backtracking and repeated rollouts that traditional Monte Carlo estimation relies on~\cite{wang2024math,luo2024improve}. 
Collecting step-level human annotations in this setup is prohibitively expensive, and even when domain-specific (process) reward models can be trained, they often fail to generalize across tasks or environments~\cite{gao2023scaling,mao2025information,shao2025spurious,zheng2025processbench}. The result is a conspicuous gap: the agents that most need process-level evaluation are precisely the ones for which building process reward models is least feasible.

\begin{figure}[!t]
    \centering
    \vspace{-0.4em}
    \makebox[\textwidth][c]{
    \includegraphics[width=1.025\textwidth]{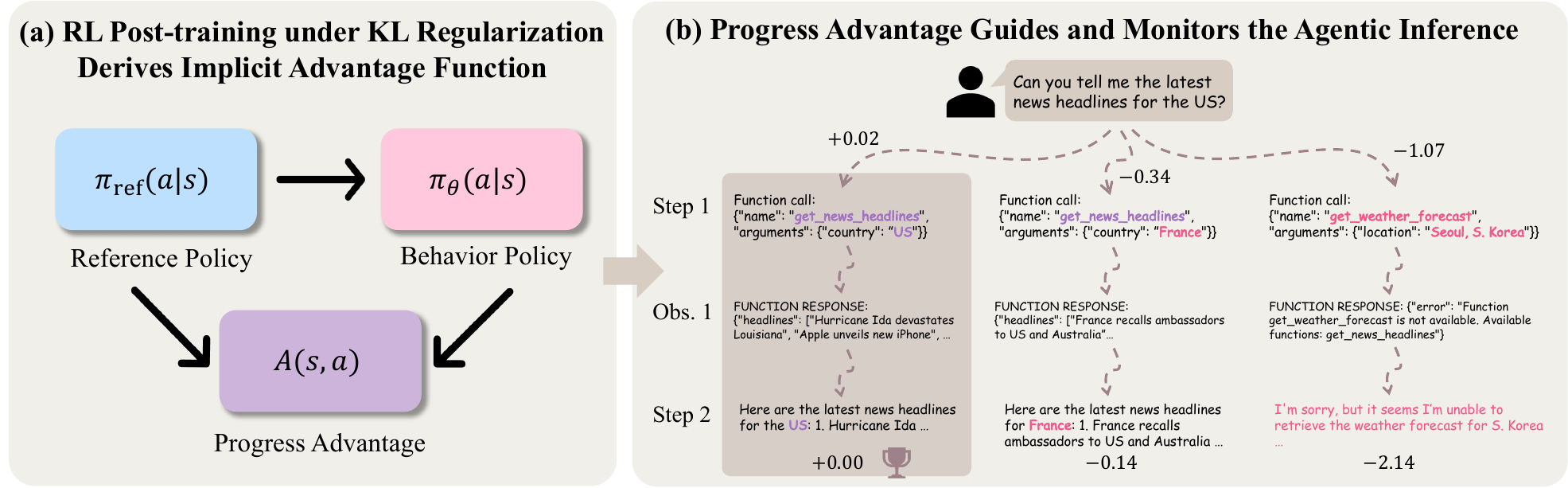}
    }
    \vspace{-1em}
    \caption{\textbf{Framework overview}. (a) We derive an optimal advantage function from an RL-trained policy and its reference policy, which can (b) score the LLM agent trajectories at both the step and trajectory levels without dedicated reward model training. 
    }
    \label{fig:illustration}
    \vspace{-0.5cm}
\end{figure}

In this paper, we take a fundamentally different approach. Rather than collecting process annotations or training dedicated reward models~\cite{lightman2023let,wang2024math,choudhury2025process,xi2025agentprm,yuan2025free,liu2025agentic}, we show that RL post-training already freely encodes a process-level signal that can be directly used for inference-time scoring.
Concretely, the log-probability ratio between the trained policy and its reference policy—readily available from standard RL post-training—constitutes a theoretically grounded measure of per-step progress,
which we term \textbf{progress advantage}.
We prove that progress advantage exactly recovers the optimal advantage function under the general stochastic MDP (Proposition~\ref{prop:prog_adv}). 
While prior implicit PRM approaches~\cite{rafailov2023direct,rafailov2024from} exploit similar likelihood-based signals in deterministic reasoning settings, agentic environments involve stochastic transitions and external interactions, where such interpretations no longer directly apply (Remark~\ref{remark}).
We show that progress advantage naturally corresponds to an advantage function for assessing an agent's sequence of actions in this setting, providing a principled and practical signal for process-level evaluation.

Notably, progress advantage has several appealing properties. It is \emph{annotation-free} and computed from checkpoint pairs that already exist as artifacts of post-training. It is \emph{general} and valid for most of mainstream RL algorithms, including those with explicit KL penalties such as GRPO~\cite{shao2024deepseekmath} and those with only clipping-based surrogates such as DAPO~\cite{yu2025dapo} (Proposition~\ref{proposition:clip_kl_solution}). It is \emph{domain-agnostic} because it emerges from the general post-training phase rather than the task-specific adaptation stage, transferring across tasks without needs for retraining. 
See Figure~\ref{fig:illustration} for an illustrative overview.

We extensively validate progress advantage across three inference-time applications on multiple agent benchmarks (BFCLv4-MT~\cite{patil2025berkeley}, WebShop~\cite{yao2022webshop}, AgentDojo~\cite{debenedetti2024agentdojo}, $\tau^2$-bench~\cite{barres2025tau}, and Who \& When~\cite{zhang25which}) and four model families (Gemma4~\cite{gemma42026modelcard},  {Qwen3.5}~\cite{qwenteam2026qwen35omni}, {Qwen3}~\cite{yang2025qwen3}, and {Olmo3}~\cite{olmo2025olmo}). In {test-time scaling}, progress advantage scores best-of-$N$ trajectory candidates to boost task success rates, outperforming confidence-based baselines, pre-trained reward models, and even task-specific PRMs (Sec.~\ref{sec:exp:tts}). In {uncertainty quantification}, it predicts trajectory-level success or failure with substantially higher AUROC than all baselines, including pre-trained PRMs or a powerful proprietary LLM-as-a-Judge baseline (Sec.~\ref{sec:exp:uq}). In {failure attribution}, it localizes the error step in multi-agent systems, approaching the step-level prediction accuracy of a method specifically trained for this task (Sec.~\ref{sec:exp:fa}). These results hold consistently across model families and benchmarks, suggesting that the signal captured by progress advantage is robust and broadly useful.

\textbf{Contribution:} \textbf{(1)} We establish the foundation of implicit reward formulation in stochastic environment and derive \textit{progress advantage} for LLM agents trained by a broad class of RL algorithms; \textbf{(2)} We demonstrate its effectiveness across three practical inference-time applications (test-time scaling, uncertainty quantification, and failure attribution), where it outperforms pre-trained reward models without any task-specific training; \textbf{(3)} We provide analyses characterizing how progress advantage works, offering practical guidance and insights for real-world adoption.

\section{Preliminary} \label{sec:prelim}
\vspace{-0.2cm}
\paragraph{Problem setup and notation.}  
We model the stochastic agents operating in the multi-turn interaction settings as a general token-level Markov Decision Process (MDP), specified by a tuple $(\mathcal{S},\mathcal{A},f,r,\rho)$. The state space $\mathcal{S}$ contains states $s_t$ representing the full sequence of tokens generated and observed up to time $t$, \emph{i.e.}, $s_t=(s_0,a_0,...,s_{t-1},a_{t-1})$. The initial state $s_0\sim\rho$ corresponds to the input prompt, such as a task specification or a user's initial query to the agent, sampled from the prompt distribution $\rho$. The action $a_t \in \mathcal{A}$ denotes the token generated at step $t$ by the agent policy $a_t\sim\pi(\cdot|s_t)$. The state transition dynamics is shaped by function $f:\mathcal{S}\times\mathcal{A}\rightarrow\mathcal{S}$ which is a stochastic transition, $s_{t+1}\sim f(\cdot|s_t,a_t)$ with a valid probability distribution ($\sum f(s_{t+1}|s_t,a_t)=1$) if $a_t$ is the end of sequence token (\texttt{EOS}) for each step. Otherwise, it is a concatenation $s_{t+1}=f(s_t,a_t)=[s_t;a_t]$. The stochastic transition brings external observations, such as user messages or tool-calling outputs, into the agent's context. Finally, a reward function $r:\mathcal{S}\times\mathcal{A}\rightarrow\mathbb{R}$ produces a scalar reward per token. The objective is to find a policy that maximizes the expected cumulative rewards. %$\sum_{i=t}^{T}r(s_i,a_i)$ given the current state $s_t$ up to the final step $T$.

\vspace{-0.2cm}
\paragraph{RL fine-tuning of LLMs.} Beyond imitation learning from expert trajectories through supervised fine-tuning (SFT), modern LLMs undergo a reinforcement learning phase, as a central component of post-training. Several popular algorithms have been proposed to realize this phase, including PPO~\cite{schulman2017proximal} and GRPO~\cite{shao2024deepseekmath}. Despite varying in their advantage estimation techniques and surrogate objectives, these methods share a common abstraction, which is a KL-regularized reward maximization problem:
{\small
\begin{align} \label{eq:kl-constrained-rl-obj}
\max_{\pi_{\theta}} J(\pi_{\theta})
=\max_{\pi_{\theta}}\mathbb{E}_{a_{t}\sim\pi_{\theta}(\cdot|s_t)}\left[\sum_t r(s_t,a_t)
-\beta  \log \frac{\pi_{\theta}(a_t|s_t)}{\pi_{\rm ref}(a_t|s_t)} \big| s_0 \sim \rho \right],
\end{align}}
\noindent where $\beta>0$ is the regularization coefficient and $\pi_{\rm ref}$ is a reference policy, typically built with pre-trained or SFT checkpoints. 

\vspace{-0.2cm}
\paragraph{Reward reparameterization in deterministic MDP.}
The objective in Eq.~\ref{eq:kl-constrained-rl-obj} admits an equivalent formulation as a maximum-entropy RL problem, which has a fixed point solution~\cite{ziebart2010modeling}, $\pi^{*}(a_t|s_t)=\exp\big( \frac{Q^{*}(s_t,a_t)-V^{*}(s_t)}{\beta} \big)$,
where $Q^*(s,a)$ is the optimal action-value function, representing the expected cumulative future reward from the state-action pair under $\pi^*$, and $V^*(s)=\mathbb{E}_{\pi^*}[Q^{*}(s,a)]$ denotes the optimal value function under $\pi^{*}$.
The \textit{Bellman equation} relates these terms recursively: 
{\small
\begin{align} \label{eq:bellman}
Q^{*}(s_t,a_t)=\begin{cases}
r(s_t,a_t)+\beta \log\pi_{\rm ref}(a_t|s_t)+\mathbb{E}_{s_{t+1}\sim f}[V^{*}(s_{t+1})], & \text{if}~s_{t+1}~\text{is not terminal} \\
r(s_t,a_t)+\beta \log\pi_{\rm ref}(a_t|s_t), & \text{otherwise.}
\end{cases}
\end{align}
\vspace{-0.5em}
}

When the MDP is deterministic, prior work~\cite{rafailov2023direct,rafailov2024from} showed that $\mathbb{E}[V^{*}(s_{t+1})]$ reduces to $V^{*}(s_{t+1})$, resulting in a clean form of implicit reward,
{\small
\begin{align} \label{eq:ir:dpo}
    r(s_t,a_t):=\beta\log\frac{\pi^{*}(a_t|s_t)}{\pi_{\rm ref}(a_t|s_t)},
\end{align}
}
\noindent where $V^{*}(s_T)=0$ for all terminal states. 
However, the deterministic MDP assumption---while sufficient for text completion-based non-interactive reasoning tasks---\emph{breaks down in multi-turn stochastic agent settings}, where external observations such as user responses, tool outputs, and environment feedback introduce non-deterministic transitions. Beyond this theoretical gap, there is also a utilization gap. Prior work has primarily leveraged implicit rewards as a {training-time} objective for policy optimization~\cite{rafailov2023direct,azar2024general,ethayarajh2024kto,hong2024orpo,meng2024simpo}, yet their potential as {inference-time} scoring signals remains underexplored.

\vspace{-0.2cm}
\section{Implicit Process Reward Modeling for LLM Agents} \label{sec:method}
Building PRMs for agents is notoriously difficult due to the long-horizon, stateful nature of the agentic environment, making step-level annotation prohibitively expensive, and Monte Carlo estimation---the standard workaround in reasoning tasks~\cite{wang2024math,luo2024improve}---becomes infeasible.
Moreover, PRMs trained for a specific task often exhibit poor generalization across different experimental setups~\cite{gao2023scaling,mao2025information}.
Here, we take a fundamentally new approach. Rather than collecting process annotations or training dedicated PRMs, we show that general RL post-training already yields effective per-token scores at inference time, which aggregate into reliable process-level signals for free.

\subsection{Implicit Rewards Under Stochastic Transitions in Agent Systems} \label{sec:method:ir}

We begin by revisiting the implicit reward formulation in Eq.~\ref{eq:ir:dpo}. An appealing property of this result is that it holds for \emph{any} policy satisfying the KL-regularized RL fixed-point condition, including models that have already undergone standard post-training. That is, given an RL-trained checkpoint $\pi^{*}$ and its reference $\pi_{\rm ref}$ (e.g., \texttt{Qwen3.5-9B} and \texttt{Qwen3.5-9B-Base}), one can readily compute per-token implicit rewards as $\log \pi^{*}(a_t|s_t) - \log \pi_{\rm ref}(a_t|s_t)$, without any reward modeling. This quantity approximates the underlying reward that was maximized during RL fine-tuning. Although this reward reparameterization has been widely adopted to guide RL fine-tuning, no existing works have studied its potential in plug-and-play post-development scenarios for LLM agents that we will explore.

Can we justify the use of the same formulation as an inference-time process reward for agents? Unfortunately, no. The clean cancellation in Eq.~\ref{eq:ir:dpo} relies on deterministic transitions, which fail to hold whenever the agent receives stochastic observations. Under a stochastic transition map $f(\cdot|s_t, a_t)$, the implicit reward acquires additional value function terms that do not cancel as below.

\begin{remark}[\textbf{Implicit Reward in Stochastic MDP}] 
\textit{The closed-form solution of KL-regularized RL and the Bellman equation under the MDP with a stochastic transition map $f(\cdot|s_t,a_t)$ results in the following token-level reward for $t=0,...,T-1$ and trajectory-level reward,
    \begin{align} 
     r(s_t,a_t)&=\beta\log\frac{\pi^{*}(a_t|s_t)}{\pi_{\rm ref}(a_t|s_t)} +V^{*}(s_{t})-\mathbb{E}_{s_{t+1}\sim f(\cdot|s_t,a_t)}[V^{*}(s_{t+1})], \label{eq:ir:sto-exact} \\
     \sum_{t=0}^{T-1} r(s_t,a_t)&=\sum_{t=0}^{T-1}\beta\log\frac{\pi^{*}(a_t|s_t)}{\pi_{\rm ref}(a_t|s_t)} +\sum_{t=0}^{T-1}\delta_{t}, \label{eq:ir:sto-exact-sum} 
\end{align}
\noindent where $\delta_{t}=V^{*}(s_{t})-\mathbb{E}_{s_{t+1}\sim f(\cdot|s_t,a_t)}[V^{*}(s_{t+1})]$ and $T$ denotes the trajectory length.}
\label{remark}
\end{remark}
Please refer to the Appendix~\ref{sec:apdx:derivation-reward} for the derivation. The residual terms $\delta_t$ in Eq.~\ref{eq:ir:sto-exact-sum} capture the discrepancy between the value of the current state and the expected value of the next state under the stochastic transition, which vanishes through telescoping sum when transitions are deterministic. Since $V^*$ is not directly accessible from the policy pair alone, the log-probability ratio no longer recovers the exact reward. This raises a natural question: if the exact reward is out of reach with only the given policy distributions, can we still extract a useful process-level signal in a stochastic MDP?

\vspace{-0.2cm}
\subsection{Progress Advantage for Agents in Stochastic World} \label{sec:method:ipr-agent}
Although the exact reward is irrecoverable from policy distributions alone, we show that a closely related and practically sufficient signal \emph{is} directly derivable. The key insight is to shift the goal: instead of recovering the absolute reward $r(s_t, a_t)$, we target the advantage function $A(s_t, a_t)$: the relative merit of taking actions $a_t$ compared to the average action under the optimal policy at state $s_t$. This turns out to have a remarkably clean form, as stated below.
\begin{proposition}[\textbf{Progress Advantage in Stochastic MDP}] \label{prop:prog_adv}
    Let $\tilde\pi^{*}$ be an optimal policy under the KL-regularized RL objective (Eq.~\ref{eq:kl-constrained-rl-obj}) with $\beta>0$, shaped with the reference policy $\pi_{\rm ref}$ where $\pi_{\rm ref}(a|s)>0$ for any $a\in\mathcal{A}$ and $s\in\mathcal{S}$. 
    Then, the optimal advantage function is exactly recovered by the log probability ratio between $\tilde{\pi}^{*}$ and $\pi_{\rm ref}$ for any state and action:
    \begin{equation} \label{eq:prog_adv}
        \tilde{A}^{*}(s,a)=\tilde{Q}^{*}(s,a)-\tilde{V}^{*}(s)=\beta\log\frac{\tilde{\pi}^{*}(a|s)}{\pi_{\rm ref}(a|s)},~~~~\forall s \in \mathcal{S}, a \in \mathcal{A}.
    \end{equation}
\end{proposition}

The proof (Appendix~\ref{sec:apdx:derivation-advantage}) follows from the Lagrangian solution of the per-state optimization induced by the soft Bellman equation. 
Critically, while environment stochasticity complicates exact reward recovery (Remark~\ref{remark}, the advantage function we derived natively isolates this by definition. Rather than canceling out the stochastic algebraically, the log probability ratio term naturally absorbs the expected future values, allowing us to extract the exact advantage $Q^{*}(s,a) - V^{*}(s)$ without requiring knowledge of the transition model.

We term this quantity \textbf{progress advantage}: it measures the expected return of taking a specific action $a_t$ relative to the average action at state $s_t$, providing a useful, fine-grained signal of whether the agent is making progress toward task completion. Fundamentally, the advantage function is the canonical quantity that drives policy improvement in reinforcement learning. The policy gradient theorem~\cite{sutton2018reinforcement} establishes that the gradient of the expected return decomposes as $\nabla_\theta J(\pi_\theta) = \mathbb{E}_{\pi_\theta}[A^{\pi_{\theta}}(s,a) \nabla_\theta \log \pi_\theta(a|s)]$---it is the advantage that determines which actions should be reinforced and which should be suppressed. Analogously, at inference time, the advantage serves as a practical state-normalized statistic for action evaluation: it tells us exactly how much better or worse a specific action is than the learned policy's expected actions, with all shared context already factored out. While the reward conflates two signals (the inherent difficulty of the state and the quality of the action taken in it),  the advantage can isolate the latter. A high reward at an easy state and a moderate reward at a hard state may reflect identical action quality; but the advantage disentangles the two. For scoring agent trajectories at inference time, this disentanglement is useful for comparing across different steps within a trajectory and across trajectories facing different environmental conditions.

\paragraph{Generality beyond explicit KL regularization.} A natural concern is whether Proposition~\ref{prop:prog_adv} applies only to methods with an explicit KL penalty (e.g., PPO~\cite{schulman2017proximal} and GRPO~\cite{shao2024deepseekmath} with adaptive KL). In practice, many widely adopted algorithms (e.g., DAPO~\cite{yu2025dapo} and Dr.\ GRPO~\cite{liu2025understanding}) instead use a clipping-based surrogate objective. We show that these methods are also covered:
\begin{proposition}[\textbf{Implicit KL Constraint of Clipping Surrogate RL}]
\label{proposition:clip_kl_solution}
Let $\pi_{\rm ref}$ and $\pi_\theta$ be the reference and target policies sharing the same support. Define the importance sampling ratio as $R(s,a) = \frac{\pi_\theta(a \mid s)}{\pi_{\rm ref}(a \mid s)}$. If optimization enforces a per-sample constraint $R(s, a) \in [1 - \varepsilon, 1 + \varepsilon]$ for all $(s, a)$ and a small $\varepsilon > 0$, then $D_{\rm KL}(\pi_{\theta} \parallel \pi_{\rm ref}) \lesssim \frac{\varepsilon^2}{2}$, similarly for reverse KL, locally at $R(s,a) \approx 1$. 
\end{proposition}
Proposition~\ref{proposition:clip_kl_solution} (proof in Appendix~\ref{sec:apdx:clip_kl_proof}) establishes that clipping-based surrogate objective~\cite{schulman2017proximal} derives a conservative KL-constrained solution policy governed by the small clipping threshold $\epsilon$.

\begin{tcolorbox}[enhanced,attach boxed title to top center={yshift=-5.5mm,yshifttext=-2.5mm},
  colback=white,colframe=gray!75!black,colbacktitle=red!80!black,
  title=,fonttitle=\bfseries,
  boxed title style={size=small,colframe=red!50!black}]
\vspace{-0.4em}
\textbf{Takeaways.} We derived progress advantage in stochastic MDP, which is computed from the optimal behavior policy and reference policy pair (Proposition~\ref{prop:prog_adv}) to measure the relative merit of each action for scoring agent trajectories. 
The progress advantage is valid for the broad class of policy trained via RL objectives with regularization, whether they contain an explicit KL regularization or a clipping-based surrogate (Proposition~\ref{proposition:clip_kl_solution}).
\vspace{-0.4em}
\end{tcolorbox}

\subsection{From Theory to Practical Implementation}\label{sec:method:practical}

Translating progress advantage into practice requires three design decisions: specifying policies, aggregating per-token advantages into process-level scores, and representing the token probability.

\vspace{-0.2cm}
\paragraph{Policy specification.} 
For $\tilde{\pi}^{*}$, one can use any policy models trained via a KL-regularized or a clipping-based  RL objective, covering virtually most mainstream post-training pipelines in use today.
Meanwhile, the selection of the reference policy $\pi_{\rm ref}$ depends on the RL pipeline adopted to get $\tilde{\pi}^{*}$: it can be a pre-trained base checkpoint in RL-Zero settings~\cite{guo2025deepseek}, an SFT checkpoint in standard single-stage RL, or a previous round's policy in online iterative RL~\cite{dong2024rlhf}. The key consideration is that $\pi_{\rm ref}$ should be neither too distant nor too close to $\tilde{\pi}^{*}$.
If too far, the log-ratio is dominated by generic distributional differences rather than task-relevant distinctions; if too close, the signal is insufficient to distinguish between good and poor actions. 
Since most model providers do not release their intermediate checkpoints, we confine our analysis to publicly available policy pairs (see Appendix~\ref{sec:apdx:exp-setup} for the full list). In Figure~\ref{fig:policymerging_uq} of Sec.~\ref{sec:exp:anal}, we further empirically analyze how the choice of reference policy affects the utility of the progress advantage.

\begin{wraptable}{r}{0.6\textwidth}
\vspace{-1.05em}
\centering
\captionof{table}{Sub-trajectory advantage aggregation over a set of consecutive token indices $\mathcal{I}$.}
\vspace{-0.65em}
\scriptsize
\resizebox{0.6\textwidth}{!}{
\begin{tabular}{@{}cl@{}}
\toprule
\textbf{Aggregation} & \textbf{Interpretation}                \\ \midrule
$\sum_{t\in \mathcal{I}} \tilde{A}^{*}(s_t,a_t)$& Vanilla sub-trajectory advantage        \\
$\frac{1}{|\mathcal{I}|}\sum_{t\in \mathcal{I}} \tilde{A}^{*}(s_t,a_t)$& Per-token average advantage \\
$\sum_{t\in \mathcal{I}} w_{t}\tilde{A}^{*}(s_t,a_t)$& Position-weighted advantage \\
$\min_{t\in \mathcal{I}}\log\tilde{\pi}^{*}-\min_{t\in \mathcal{I}}\log\pi_{\rm ref}$ (resp. $\max$) & Extreme token advantage \\ 
\bottomrule
\end{tabular}\label{tab:aggregation}
}
\vspace{-0.975em}
\end{wraptable}
\paragraph{Progress advantage aggregation.} Since Eq.~\ref{eq:prog_adv} produces a token-level advantage $\tilde{A}^{*}(s_t, a_t)$ at each position $t$, we need aggregation strategies to obtain step-level and trajectory-level scores suited to each application. Table~\ref{tab:aggregation} shows some natural choices. Simple summation yields the standard additive trajectory advantage, while averaging produces a length-normalized variant that prevents long trajectories from being scored higher. In addition, one can inject an inductive bias based on their knowledge to implement position-weighted advantage. Extreme token advantage (min or max) captures the worst or best-case token advantage within a sub-trajectory. 
The aggregation choice can meaningfully affect the quality of progress advantage (Figure~\ref{fig:agg_heatmap}), and we pick the best per task. 

\paragraph{Representing the token probability.} The clean implementation of Eq.~\ref{eq:prog_adv} is directly using $\tilde{\pi}^{*}(a_t|s_t)$ (and $\pi_{\rm ref}(a_t|s_t)$), but the pure token probability is often noisy and a source of instability during RL~\cite{tang2025few,qi2026rethinking}. Thus, we also explore a top-$k$ average token probability variant in Appendix~\ref{sec:apdx:progadv} and~\ref{sec:apdx:exp-result}.

\begin{table}[t]
\caption{\textbf{Test-time scaling through best-of-8 sampling}. We compare reward scoring methods across four benchmarks and two LLM backbones, reporting the success rate (\%) of the selected trajectory. Our progress advantages successfully boost the success rate, especially when the non-zero-temperature exploratory behavior is beneficial to the task, i.e., WebShop and $\tau^2$-Airline.}
\label{tab:tts-bon-main}
\centering
\resizebox{1.0\textwidth}{!}{%
\begin{tabular}{lc|cc|cc|cc|cc|cc}
\toprule
Scoring Method     & Training & \multicolumn{2}{c}{BFCLv4-MT} & \multicolumn{2}{|c}{WebShop} & \multicolumn{2}{|c}{AgentDojo} & \multicolumn{2}{|c}{$\tau^{2}$-Airline} & \multicolumn{2}{|c}{Average} \\ \midrule
     & & \texttt{Gemma4-4B} & \texttt{Qwen3.5-9B} & \texttt{Gemma4-4B} & \texttt{Qwen3.5-9B}  & \texttt{Gemma4-4B} & \texttt{Qwen3.5-9B} & \texttt{Gemma4-4B}    & \texttt{Qwen3.5-9B} & \texttt{Gemma4-4B}   & \texttt{Qwen3.5-9B}   \\ \midrule
\rowcolor{mygray} 
Pass@N (oracle)    & \xmark& 22.0 & 48.0 & 53.0 & 49.0 & 48.5 & 94.8 & 58.0 & 78.0 & 45.4 & 67.5 \\ 
\rowcolor{mygray} 
Greedy Decoding     & \xmark& 19.0 & 42.5 & 32.0 & 21.0 & 48.5 & 94.8 & 34.0 & 60.0 & 33.4 & 54.6 \\ 
Mean-of-N & \xmark& 17.5 & 38.5 & 41.6 & 26.4 & 38.8 & \underline{89.0} & \underline{34.5} & \underline{64.8} & 33.1 & 54.7 \\ \midrule
WildReward-8B~\cite{peng2026wild} & \cmark & \textbf{20.0} & \textbf{42.5} & 41.0 & 26.0 & \underline{43.3} & 86.6 & 28.0 & 64.0 & 33.1 & 54.8 \\
ThinkPRM-7B~\cite{khalifa2025process} & \cmark & 18.5& 37.0 & 38.0 & 22.0  & 37.1 & 85.6 & 30.0 & 64.0 & 30.9 & 52.2 \\
ThinkPRM-14B~\cite{khalifa2025process} & \cmark & \underline{19.0} & \underline{40.0} & \underline{43.0} & \underline{33.0} & \textbf{44.3} & 88.7 & 28.0 & 58.0 & \underline{33.6} & {54.9} \\ \midrule
Self-Certainty~\cite{kang2025scalable} & \xmark & 15.0 & 34.5 & 34.0 & 22.0 & 33.0 & 85.6 & 34.0 & 64.0 & 29.0 & 51.5 \\
DeepConf Tail~\cite{fu2026deep}   & \xmark & 15.5 & 36.0 & 39.0 & 28.0 & 35.1 & 86.6 & 36.0 & 62.0 & 31.4 & 53.2 \\
DeepConf B10~\cite{fu2026deep}  & \xmark & 15.5 & 34.5 & 35.0 & 30.0 & 30.9 & 86.6 & 28.0 & \textbf{72.0} & 27.4 & \underline{55.8} \\ \midrule
\rowcolor{lightpurple}
Progress Advantage & \xmark& \underline{19.0} & \textbf{42.5} & \textbf{45.0} & \textbf{42.0} & \underline{43.3} & \textbf{91.8} & \textbf{48.0} & \textbf{72.0} & \textbf{38.8} & \textbf{62.1} \\ \bottomrule
\end{tabular}}%}
\vspace{-1.25em}
\end{table}
\section{Empirical Validation} \label{sec:exp}
We evaluate progress advantage on three inference-time applications that collectively test whether the signal is useful for (1) parallel test-time scaling through best-of-N sampling, (2) uncertainty quantification, and (3) failure attribution. Appendix~\ref{sec:apdx:exp-setup} and \ref{sec:apdx:exp-result} cover additional details and results.
\vspace{-0.2cm}
\subsection{Setup} \label{sec:exp:setup}
\paragraph{Benchmarks.} We ground our evaluation in four benchmarks that represent realistic agentic workloads: BFCLv4-MT~\cite{patil2025berkeley} (multi-turn tool calling), WebShop~\cite{yao2022webshop} (tool-augmented online shopping), AgentDojo~\cite{debenedetti2024agentdojo} (tool-augmented general task solving), and $\tau^{2}$-bench~\cite{yao2024tau,barres2025tau} (conversational agents in customer-service environments). All four require multi-turn interaction with external tools in stateful environments, exercising precisely the stochastic MDP structure that motivates our approach. 
Each application defines a distinct evaluation protocol. For test-time scaling, we perform best-of-$N$ sampling by generating 8 trajectories per task, with temperature $0.7$ for $\tau^{2}$-bench and WebShop, and 0.4 for the remaining. We score these trajectories with each reward method and measure the average task success rate of the selected trajectories. For uncertainty quantification, we use trajectory-level reward to predict whether each trajectory succeeds or fails on $\tau^2$-bench, measured by AUROC~\cite{oh2026uncertainty}. Failure attribution serves a step-level failure detection, which is described separately in Sec.~\ref{sec:exp:fa}.

\vspace{-0.2cm}
\paragraph{Models.} We mainly evaluate four public LLM families: \texttt{Gemma4-4B}~\cite{gemma42026modelcard},  \texttt{Qwen3.5-9B}~\cite{qwenteam2026qwen35omni}, \texttt{Qwen3-14B}~\cite{yang2025qwen3}, and \texttt{Olmo3-7B}~\cite{olmo2025olmo}. For each, we pair the RL-trained final checkpoint with its corresponding base/intermediate checkpoint as the reference policy to build progress advantage.

\vspace{-0.2cm}
\paragraph{Baseline method.} We compare against two categories of baselines. \emph{Trained reward models}: (1) WildReward-8B~\cite{peng2026wild}, (2) ThinkPRM-7B/14B~\cite{khalifa2025process}, which are specifically trained on real-world user-agent interactions or multi-step reasoning datasets, and (3) AgentPRM~\cite{xi2025agentprm} that is specifically trained on a downstream task. \emph{Confidence-based methods}: (4) Self-Certainty~\cite{kang2025scalable}, which scores trajectories by the average token probability certainty, and (5) DeepConf~\cite{fu2026deep}, which proposes step-level confidence aggregation strategies including tail-step confidence and bottom-10\% average of step confidences.
Crucially, progress advantage and the confidence baselines require {no dedicated training}, whereas trained reward models usually require task-specific or domain-specific supervision. 
\subsection{Test-time Scaling} \label{sec:exp:tts}
We start with best-of-8 sampling scenarios in Tab.~\ref{tab:tts-bon-main}, where Pass@N denotes the pass rate of at least one of trajectories, Mean-of-N denotes the mean success rate of them, and Greedy Decoding is a zero-temperature deterministic generation. Overall, progress advantage shows stable performance across datasets and models, outperforming the expensive training-based method and confidence-based methods with significant margins (15.5\% for \texttt{Gemma4} and 11.3\% for \texttt{Qwen3.5}) on average. Notably, it consistently beats the baseline methods in cases where high-temperature exploratory trajectories are beneficial over greedy ones, e.g., WebShop and $\tau^{2}$-Airline. We hypothesize that optimal advantage signals favor setups where both the average (Mean-of-N) and ceiling (Pass@N) values of trajectories are high, connected to the theorem. Table~\ref{tab:tts-training-based-rm} in Appendix~\ref{sec:apdx:exp-result} further shows that it even outperforms AgentPRM-7B, specifically trained on a downstream task.

\subsection{Uncertainty Quantification} \label{sec:exp:uq}
\vspace{-0.1em}
\begin{table}[th!]
\caption{\textbf{Uncertainty quantification for trajectory monitoring}. We compare scoring methods across four LLM backbones on $\tau^2$-bench Airline and Retail domains to predict an agent's success on each model's greedy-decoding trajectory with trajectory-level scoring, measured by AUROC. 
}
\label{tab:uq-on-main}
\centering
\resizebox{\textwidth}{!}{%
\begin{tabular}{lc|cccc|cccc}
\toprule
Scoring Method     & Training & \multicolumn{4}{c}{$\tau^{2}$-Airline}                                          & \multicolumn{4}{|c}{$\tau^{2}$-Retail}                                           \\  \midrule
                   &          & \texttt{Gemma4-4B}  & \texttt{Qwen3.5-9B} & \texttt{Qwen3-14B} & \texttt{Olmo3-7B}   & \texttt{Gemma4-4B}  & \texttt{Qwen3.5-9B} & \texttt{Qwen3-14B} & \texttt{Olmo3-7B}   \\ \midrule
\rowcolor{mygray}{\claudelogo} Sonnet-4.6~\cite{anthropic2026sonnet46}      & \xmark        & 0.615          & 0.726            & 0.519           & 0.715          & 0.852    & 0.899            & 0.864  & 0.656          \\
WildReward-8B~\cite{peng2026wild}      & \cmark        & 0.312          & 0.540            & 0.314           & 0.514          & \underline{0.643}    & 0.468            & \underline{0.689}  & 0.584          \\
ThinkPRM-7B~\cite{khalifa2025process}        & \cmark        & 0.478    & 0.582            & 0.276           & 0.492          & 0.469          & 0.551            & 0.543           & \textbf{0.670} \\
ThinkPRM-14B~\cite{khalifa2025process}       & \cmark        & 0.426          & \underline{0.655}      & 0.292           & \underline{0.708}    & 0.573          & \underline{0.610}            & 0.637           & 0.544          \\ \midrule
Self-Certainty~\cite{kang2025scalable}     & \xmark        & \underline{0.840}          & 0.642            & 0.663           & 0.486          & 0.397          & 0.366            & 0.608           & 0.392       \\
DeepConf Tail~\cite{fu2026deep}      & \xmark        & 0.581          & 0.588            & \underline{0.682}           & 0.472          & 0.382          & 0.344            & 0.380           & 0.608         \\
DeepConf B10~\cite{fu2026deep}       & \xmark        & 0.834          & 0.587            & 0.636           & 0.618          & 0.416          & 0.496            & 0.582           & 0.288        \\ \midrule
\rowcolor{lightpurple}
\rowcolor{lightpurple}
Progress Advantage & \xmark        & \textbf{0.865} & \textbf{0.720}   & \textbf{0.739}     & \textbf{0.799} & \textbf{0.690} & \textbf{0.678}      & \textbf{0.715}           & \underline{0.664}    \\ \bottomrule
\end{tabular}}
\vspace{-1.4em}
\end{table}
One of the key building blocks for a reliable agent in the wild is a framework for uncertainty quantification. As noted by Oh et al.~\cite{oh2026uncertainty}, quantifying uncertainty in a multi-turn interactive inference setup brings non-trivial open problems hard to tackle with existing uncertainty methods. This subsection explores the application of the progress advantage for the UQ of LLM agents. Specifically, we predict whether a trajectory generated by an agent ends with success or not by adopting the trajectory-level reward as a (un)certainty signal. Table~\ref{tab:uq-on-main} shows AUROC computed over the whole trajectory samples on $\tau^{2}$-bench (50 for Airline, 114 for Retail) across four different models, where we score the trajectory generated by each behavior model through its own log probability with (ours) and without (Self-Certainty and DeepConf) reference policy's log probability offset, or a different, pre-trained reward model. We see that the progress advantage remarkably outperforms all the baselines in $\tau^2$-Airline and also shows competitive results on $\tau^2$-Retail, demonstrating its validity under a stochastic MDP under the complex interaction scaffolding (See Figure~\ref{fig:advantage_evo} for more analyses).

\vspace{-0.3em}
\begin{wraptable}{r}{0.54\textwidth}
\vspace{-1.25em}
\caption{\textbf{Uncertainty quantification on trajectories generated by a different policy}. 
\texttt{Gemma4-4B} as a reward model scoring trajectories produced by different behavior policies on $\tau^2$-Airline. Trajectory-level AUROC over success prediction (higher is better).
}
\label{tab:uq-off-qwengemma}
\centering
\scriptsize
\begin{tabular}{lcc}
\toprule
Scoring Method     & \texttt{Qwen3.5-9B} & \texttt{Qwen3-14B} \\ \midrule
Self-Certainty~\cite{kang2025scalable} & \underline{0.587} & \underline{0.648}  \\
DeepConf Tail~\cite{fu2026deep} & 0.482 & 0.610 \\
DeepConf B10~\cite{fu2026deep} & 0.563 & 0.636  \\ \midrule
\rowcolor{lightpurple}
Progress Advantage & \textbf{0.754} & \textbf{0.727}  \\ \bottomrule
\end{tabular}
\vspace{-0.9em}
\end{wraptable}
We further test whether the progress advantage can predict the chance of success over a trajectory generated by another policy backbone model. In Table~\ref{tab:uq-off-qwengemma}, we use \texttt{Gemma4-4B} as a reward model to score trajectories of \texttt{Qwen3.5-9B} and \texttt{Qwen3-14B}  on $\tau^{2}$-Airline. We see that the progress advantage acts as an external scorer to assess the quality of the action sequence yielded by a different policy, implying its potential as an off-the-shelf reward model to monitor arbitrary trajectories.

\subsection{Failure Attribution} \label{sec:exp:fa}
\vspace{-0.2em}

\begin{wrapfigure}{r}{0.56\textwidth}
    \vspace{-0.9em}
    \centering
    \includegraphics[width=\linewidth]{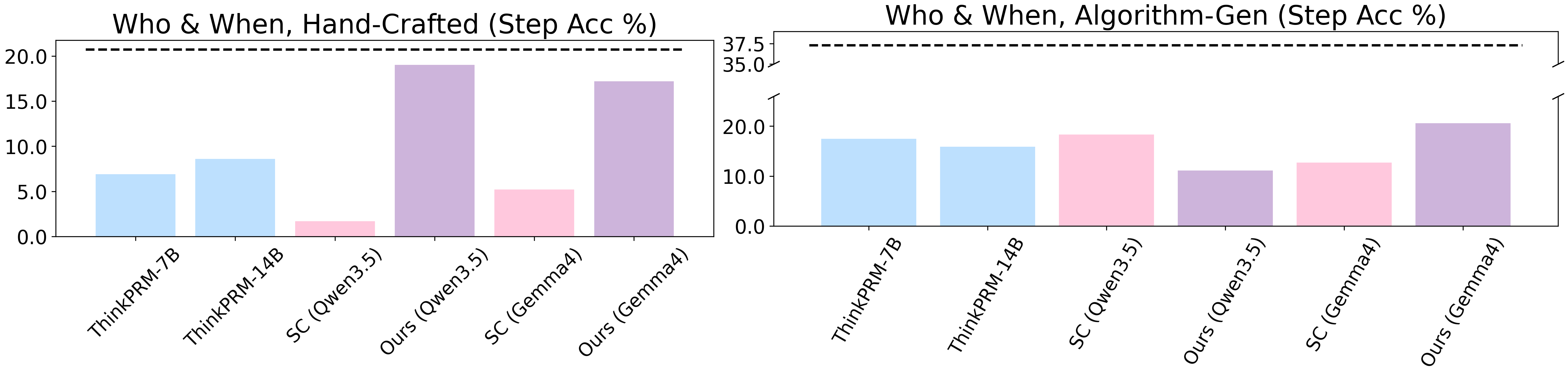}
    \vspace{-1.05em}
    \caption{\textbf{Who \& When step-level accuracy}. 
    We predict when the agent system makes a decisive error. SC denotes Self-Certainty~\cite{kang2025scalable}, and the dashed line denotes AgenTracer~\cite{zhang2025agentracer}, which is specifically trained on this.
    }
    \label{fig:fa-main}
\end{wrapfigure}
An emerging field of agentic system monitoring is \textit{failure attribution}, where we detect a step when the system would make the critical error across the whole trajectory. We evaluate PRMs on Who \& When benchmark~\cite{zhang25which}, predicting the decisive error step, $t_{\rm err}$, over pre-extracted trajectories from multi-agent systems. Here, we make a prediction as an index of the minimum per-step reward, and compare the prediction with the ground truth, i.e., $\mathbb{I}(\arg\min_t \hat{r}_t=t_{\rm err})$. As shown in Figure~\ref{fig:fa-main}, the task is notably challenging for pre-trained reward models and even for a task-specific RL-trained baseline, AgenTracer~\cite{zhang2025agentracer}. Our method shows promising results on both splits, while rivaling AgenTracer in the Hand-Crafted split, demonstrating its reliable step-level credit assignment under a carefully designed agentic harness. 
\vspace{-0.2em}

\subsection{Additional Empirical Study on Progress Advantage} \label{sec:exp:anal}
%\vspace{-0.3em}
\begin{wraptable}{r}{0.3\textwidth}
\vspace{-1.5em}
\caption{\textbf{Progress advantage and its ingredients on UQ.}}
\vspace{-0.65em}
\label{tab:component_ablation_reduced}
\centering
\tiny
\resizebox{\linewidth}{!}{
\begin{tabular}{l|cc}
\toprule
Method & Avg. Rank \\ \midrule
\rowcolor{lightpurple}
Ours & \textbf{1.44} \\
$\log\tilde{\pi}^{*}(a|s)$   & 2.25 \\
$\log\pi_{\rm ref}(a|s)$ & 2.31 \\ \bottomrule
\end{tabular}
}
\vspace{-1.35em}
\end{wraptable}
\paragraph{Contrasting $\tilde{\pi}^{*}$ and $\pi_{\rm ref}$ is better than sole.} Since the progress advantage is defined with two policies, one may wonder if we can just use one of them as a reward with the same aggregation strategies. Tab.~\ref{tab:component_ablation_reduced} provides the average rank of AUROC on $\tau^{2}$-Airline uncertainty quantification (See Tab.~\ref{tab:component_ablation} for setup and more results), confirming that progress advantage provides more reliable signals sharpened by contrasting distributions~\cite{li2023contrastive}. We go into this in Fig.~\ref{fig:qual_anal}.

\paragraph{Per-token qualitative analysis.} 
We perform fine-grained analysis to investigate whether the progress advantage produces reasonable signals related to goal achievement. Figure~\ref{fig:qual_anal} presents a case in which the agent correctly refuses a reservation cancellation request in conflict with the domain-specific constraint in $\tau^{2}$-Airline. The results reveal a clear contrast between pure policy log probability $\log\tilde{\pi}^{*}(\cdot|\cdot)$ and progress advantage $\log\frac{\tilde{\pi}^{*}(\cdot|\cdot)}{\pi_{\rm ref}(\cdot|\cdot)}$. Specifically, the pure policy log probability assigns low scores to tool-calling strings (steps 1 and 2) even though they are correct, probably because the frequency of the tool string is lower than plain natural language.
In contrast, progress advantage assigns positive scores to these strings, thanks to the offset effect of the reference log probability.
Furthermore, the policy log probability penalizes domain constraint-related terms in step 3, such as ``change of plan'' and ``business class'' which cover the key criteria of canceling a flight. Progress advantage, on the other hand, rewards most of these terms, demonstrating its awareness of goal-specific information to effectively reward actions that induce success on the task.
\begin{figure}
    \centering
    \vspace{-0.65em}
    \includegraphics[width=\linewidth]{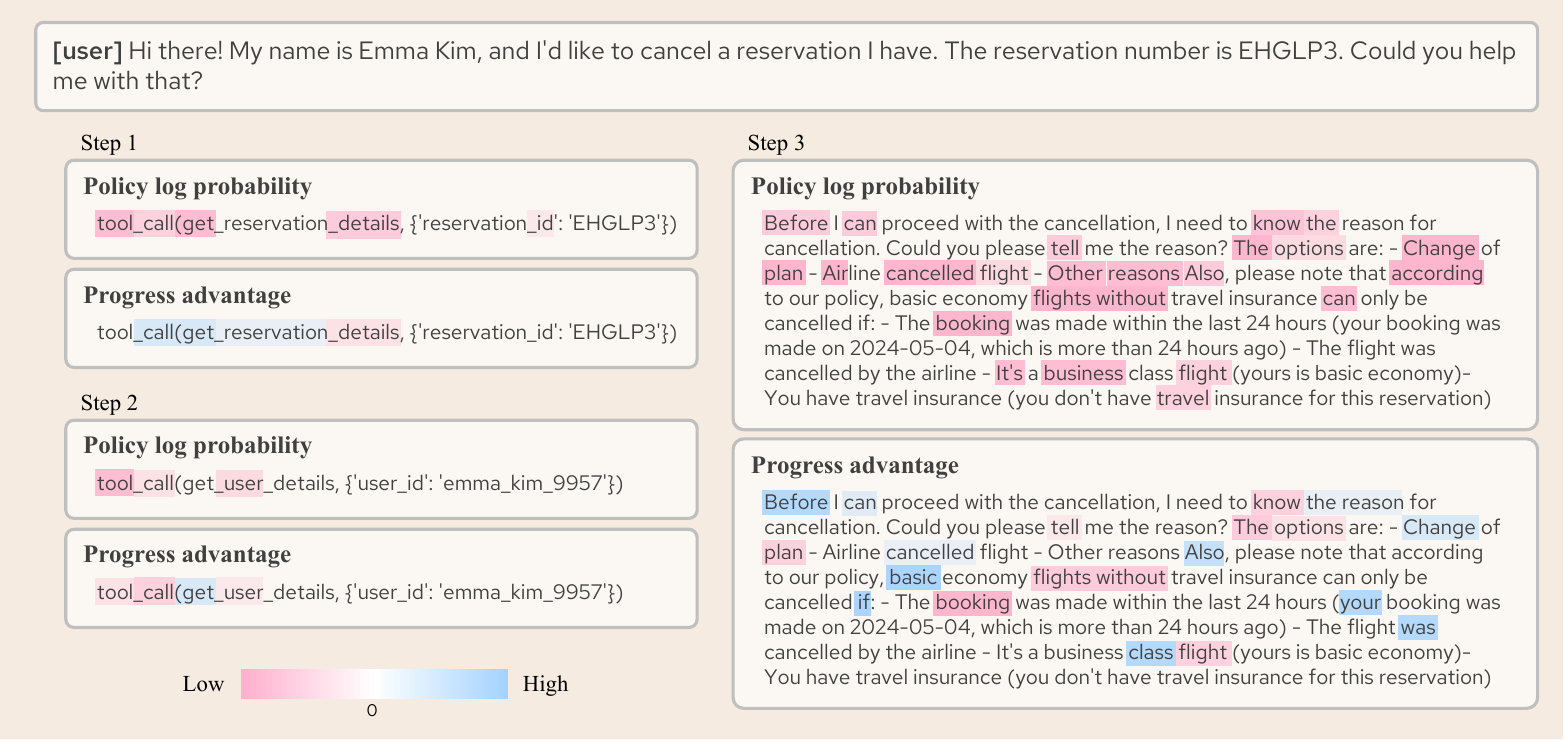}
    \vspace{-1em}
    \caption{\textbf{Qualitative analysis on token-level signals.} Progress advantage effectively rewards actions specifically helpful to achieve the downstream goal, whereas the policy log probability does not.}
    \label{fig:qual_anal}
    \vspace{-1.2em}
\end{figure}

\paragraph{Advantage aggregation strategy.}
\begin{figure*}[htb]
    \centering
    \vspace{-1em}
    \makebox[\textwidth][c]{
    \includegraphics[width=1.025\linewidth]{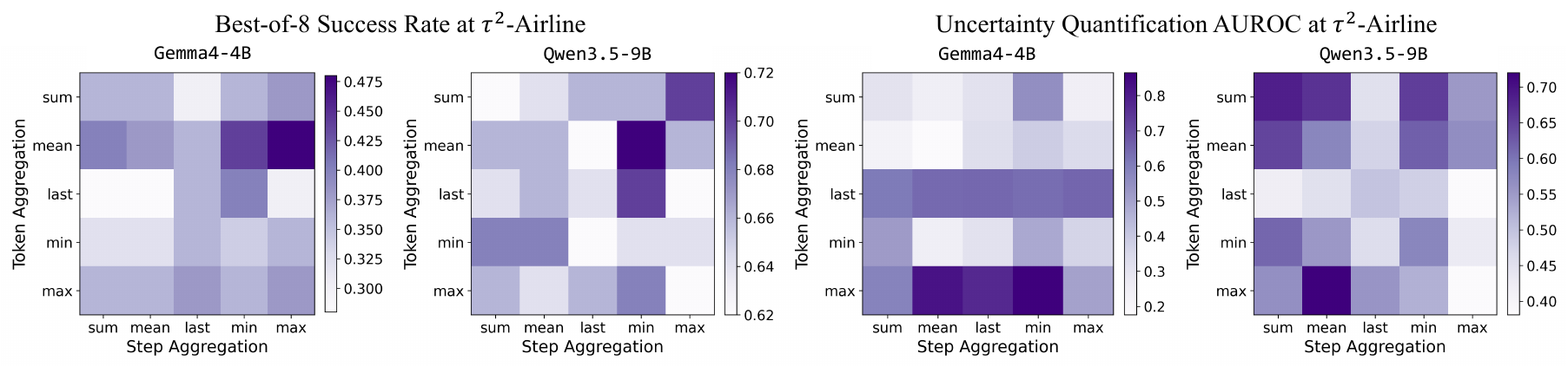}}
    \vspace{-1.0em}
    \caption{\textbf{Combinations of token and step aggregation strategy for progress advantage.} The aggregation across token and step advantages affects the effectiveness of progress advantage, and each downstream task and model shows quite a different flavor in the aggregation strategy.}
    \label{fig:agg_heatmap}
    \vspace{-0.5em}
\end{figure*}

Since our progress advantage (Eq.~\ref{prop:prog_adv}) serves as a token-level signal, we explore aggregation strategies at both token and step levels, following prior work on sentence and trajectory-level scores~\cite{zhang2023enhancing,duan2024shifting,fu2026deep}.
As shown in Fig.~\ref{fig:agg_heatmap}, different combinations excel at different applications: the (\textsc{mean}, \textsc{min}) pair performs strongly for best-of-$N$ selection, while the (\textsc{max}, \textsc{mean}) pair becomes the winner for UQ.
One possible explanation is that step-level \textsc{min} aggregation penalizes trajectories containing a low-quality step, thereby favoring trajectories whose progress signal remains consistently positive across the interaction. This can be ensure a better sequence of actions during the negotiation-centric airline tasks.
For UQ, on the other hand, focusing on the extrema, such as maximum token advantage, can be a better indicator for per-step success, which may capture salient local evidence crucial for ultimate success. Then, the conventional mean operation over per-step maximum advantages can reliably stand for per-trajectory uncertainty. This is aligned with findings from the reasoning model inference, where some important tokens drive the final success~\cite{qian2026demystifying,hwang2026oops}. Extended results are provided in Fig.~\ref{fig:agg_heatmap_full_tts} and \ref{fig:agg_heatmap_full_uq}, and we show that even universal aggregation without per-setting tuning still outperforms baselines in Tab.~\ref{tab:agg-fixed-rank}.

\begin{wrapfigure}{r}{0.35\textwidth}
    \vspace{-1em}
    \centering
    \includegraphics[width=\linewidth]{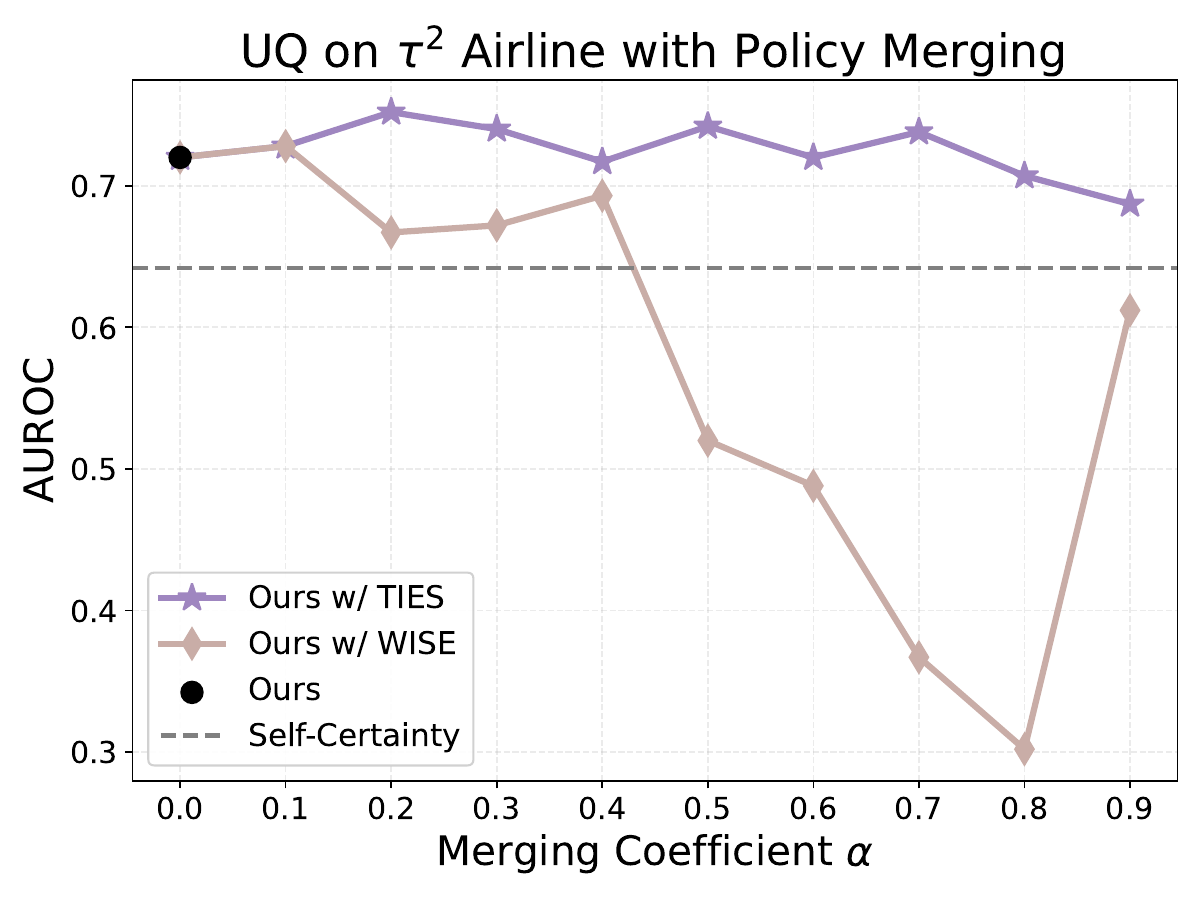}
    \vspace{-1em}
    \caption{\textbf{Varying reference policy.} We merge \texttt{Qwen3.5-9B-Base} with \texttt{Qwen3.5-9B} in the weight space and use it as $\pi_{\rm ref}$ in our progress advantage for $\tau^2$-Airline UQ.}
    \label{fig:policymerging_uq}
    \vspace{-0.85em}
\end{wrapfigure}
\paragraph{Specification of reference policy.} As noted in Sec.~\ref{sec:method:practical}, progress advantage is constructed with the behavior policy and the reference policy, and the reference policy specification becomes a design choice. Rather than simply adopting the base checkpoint version of the final policy, we test policy merging between the final and base checkpoints, $\theta_{\alpha}=\alpha \theta_{\rm final}+(1-\alpha)\theta_{\rm base}$, to get a spectrum of reference policies $\pi_{\theta_\alpha}$ for $\alpha\in\{0.1,...,0.9\}$ to construct $\log\frac{\tilde{\pi}^{*}(a|s)}{\pi_{\theta_\alpha}(a|s)}$ in Figure~\ref{fig:policymerging_uq}. Here, we adopt two different merging methods, WISE~\cite{Wortsman_2022_CVPR} as the aforementioned simple linear interpolation and TIES~\cite{yadav2023ties} as an interference-aware robust merging variant. While naive linear interpolation mostly degrades the quality of progress advantage, the robust merging variant consistently induces better results in most $\alpha$ from 0.2 to 0.7. This supports our hypothesis that the reference should not be too far or too close to the behavior policy $\tilde{\pi}^{*}$, suggesting the promise of progress advantage coupled with advanced merging methods~\cite{ortiz2023task,yang2024adamerging,yu2024language,jang2024model,oh2025dawin} to craft a sharper reference policy.

\begin{wrapfigure}{r}{0.45\textwidth}
\vspace{-1.2em}
\centering
\includegraphics[width=0.45\textwidth]{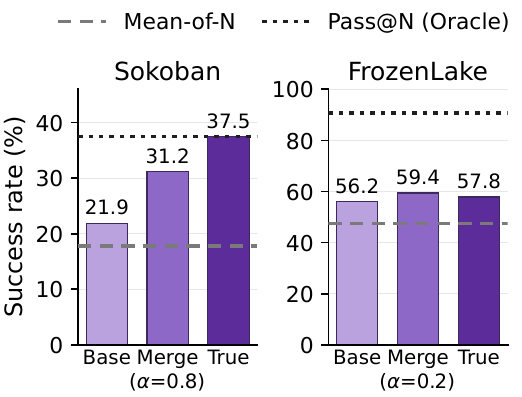}
\vspace{-1.6em}
\caption{\textbf{Effect of reference mismatch in a controlled setup.} Best-of-128 success rate of progress advantage of (\textsc{mean}, \textsc{mean}) aggregation with three references: the base proxy (\texttt{Qwen2.5-3B}), a linearly merged reference, and the true reference (\texttt{Qwen2.5-3B-Instruct}).}
\label{fig:true-ref}
\vspace{-1em}
\end{wrapfigure}

To further investigate how much the mismatch between true and proxy reference policy costs, we run agentic RL in a controlled setting where the true reference is known. Following RAGEN~\cite{wang2025ragen}, we train \texttt{Qwen2.5-3B-Instruct} (the reference) with StarPO-S with PPO as its backbone RL algorithm for 200 steps on two puzzle game environments: Sokoban, which has deterministic transitions, and FrozenLake, where the intended move succeeds only with probability $0.8$. 
For each of 64 held-out seeds, we sample 128 rollouts and select one of them with progress advantage, using one of three references: the true reference, the base proxy, or the WISE-based merged reference between the two as we did above.

On Sokoban, the true reference matches the oracle: it selects a successful rollout for every seed. Meanwhile, the base proxy achieves a modest improvement, indicating that reference mismatch can matter. However, it still improves over random selection by 4.1 points. Moreover, linear merging recovers about 60\% of the gap between the base proxy and the true reference, using only the base and final endpoint checkpoints.
On FrozenLake, all three references achieve similar improvement over the random selection baseline, while the gap towards oracle is large due to stochastic transition noise. Overall, an exact reference is preferable whenever available; base and merged proxies still deliver useful signals. 
%We therefore present progress advantage, as used in practice, as an \emph{approximate} advantage signal for cases where the exact conditions of Proposition~\ref{prop:prog_adv} do not hold.
\vspace{-0.4em}
\section{Related Work} \label{sec:related_work}
\paragraph{Process Reward Modeling for Reasoning Models.} PRMs provide fine-grained supervision over intermediate reasoning steps, assisting models to improve their reasoning quality step-by-step. Early works typically formulate process supervision as a binary classification, where each step is labeled as correct or incorrect~\cite{lightman2023let,snell2025scaling}. Recent approaches move beyond step-wise classification; instead cast PRM learning as a ranking problem, drawing on Q-value theory~\cite{li2025process} or probabilistic formulations of step quality~\cite{zhang2025linking}. In parallel, because annotating reasoning trajectories at the step-level is expensive and time-consuming, several works explore efficient or automated strategies to construct step supervision~\cite{luo2024improve,wang2024math,lee2026efficient}. Notably, implicit PRMs~\cite{yuan2025free} eliminate the need for explicit step annotations during training has been shown to be effective for both test-time scaling and RL of reasoning models~\cite{cui2025process}. However,  these methods assume deterministic and non-interactive text completion for multi-step reasoning; our work differs by deriving an implicit process reward that is theoretically grounded in stochastic MDPs and validated in multi-turn agentic settings.
\paragraph{Process Reward Modeling for LLM Agents.} 
There are emerging research endeavors to build PRMs for LLM agents that use tools and interact with the user and environment to achieve a long-horizon goal.
Prior work typically relies on per-step MC estimation~\cite{choudhury2025process,xi2025agentprm} to get step-level annotations while confining the scope to relatively simple tasks with a few rounds of turns, which is unreliable~\cite{zhang2025lessons,zeng2025versaprm} and becomes infeasible in complex, long-horizon tasks.
Although Liu et al.~\cite{liu2025agentic} bypass the process supervision by adopting an implicit reward formulation, they still perform the downstream task-specific training with outcome-level supervision to learn the implicit PRMs. In contrast, we showed that an RL-trained policy, paired with some reference policy, constructs an implicit advantage function that already can be a sufficient signal to guide test-time scaling or post-deployment monitoring, \emph{without any extra training} (See Appendix~\ref{sec:apdx:context} for a broader context).
\vspace{-0.4em}
\section{Conclusion} \label{sec:conclusion}

LLMs equipped with the agentic harness operate on multi-turn, interactive environments under stochasticity, where measuring the intermediate progress over a goal-oriented trajectory brings huge opportunities; at the same time, the real bottleneck---expensive training of PRMs with process annotations.
We establish the theoretical foundation through progress advantage for implicit reward technology in stochastic MDP; offer a novel angle and recipe to build PRMs from LLM checkpoint pairs of RL fine-tuned and its base, bypassing collecting the process labels and dedicated reward model training.
Across five agentic benchmarks, four model families, and three downstream applications, progress advantage consistently outperforms confidence-based baselines and matches or beats competitive training-based reward models as well as a proprietary LLM judge.
We hope these promising results stimulate a new paradigm of future research towards a scalable and practical approach to process-level guidance and monitoring of real-world agentic systems.

\paragraph{Limitation and Future Work.} %\label{sec:apdx:limitation}
Since institutions usually do not release all intermediate checkpoints from their LLM development pipelines, we inevitably limited the scope of model candidates in our experiments. We hope that the community pursues the fully open model development cycle in the future for novel uses like our progress advantage. 
Also, although we assumed that the public, post-trained models are close to the optimal solution of the RL objective, it is barely falsifiable since we usually don't know the full training configuration details, as well as the training log. Future work with controlled experiments to validate the approximation quality of progress advantage would be appealing.
Besides, although we established a theoretical foundation on the implicit process reward for agents in stochastic MDPs, we leave vague the engineering efforts to find the best working specifications and implementations of progress advantage for future work.
Another possible future exploration is to expand the domain beyond general assistant LLM agents. Since our progress advantage formulation does not make any assumptions about data domain or modality, one may be able to adopt it as a baseline fine-grained reward signal given any RL-trained policy and its base pair. In that sense, test-time scaling and runtime monitoring/intervention on multimodal agents~\cite{wang2026safeground,lee2025streamgaze,wang2026leash}, VLA models~\cite{kwok2025robomonkey,park2026hide}, embodied agents~\cite{fang2026empc,son2025subtle}, and user simulator agents~\cite{li2026thinking,zhou2026mind,tack2026llms,phan2026mimesis} would be worthwhile extensions of the progress advantage.

\begin{ack}
We sincerely thank Jiatong Li, Leitian Tao, Sangyun Lee, and Jiaying Fang for their faithful proofreading and professional feedback on the draft that directly affected the writing and experiment content and sparked future work ideas. 
This material is based upon work supported by the U.S. Department of Energy, Office of Science, under contract number DE-AC02-06CH11357. Wendi Li, Seongheon Park, Samuel Yeh, and Sharon Li are also supported in part by the AFOSR Young
Investigator Program under award number FA9550-23-1-0184, the Office of Naval Research (ONR) under award number N000142612508, National Science Foundation under
awards IIS-2237037 and IIS-2331669, 
Schmidt Sciences Foundation, Open Philanthropy (now Coefficient Giving), Alfred P. Sloan Fellowship, and gifts
from Google and Amazon.
\end{ack}

{
\small
\bibliographystyle{unsrt}
\bibliography{main}
}

\newpage
\appendix
\begin{center}
    \LARGE \textbf{Appendix}
    \vspace{1em}
\end{center}
\noindent\rule{\linewidth}{0.4pt}
\vspace{0.5em}
\tableofcontents
\noindent\rule{\linewidth}{0.4pt}
\vspace{0.5em}
\addtocontents{toc}{\protect\setcounter{tocdepth}{2}}
\clearpage
\newpage
\section{Implementation of Progress Advantage} \label{sec:apdx:progadv}
\subsection{Policy Pairs in Progress Advantage} \label{sec:apdx:progadv:pair}
\begin{table}[ht!]
\caption{\textbf{Policy pair lineup used for progress advantage construction.} We consider the following five open-source model families, which offer the base or intermediate checkpoint of their final post-trained version.}
\label{tab:policypair}
\centering
\resizebox{1.0\textwidth}{!}{
\begin{tabular}{@{}llll@{}}
\toprule
Model Family & Behavior Policy $\tilde{\pi}^{*}$   & Reference Policy $\pi_{\rm ref}$       & HuggingFace Collection URL                                               \\ \midrule
Qwen3.5~\cite{qwenteam2026qwen35omni}      & Qwen3.5-9B          & Qwen3.5-9B-Base        & \url{https://huggingface.co/collections/Qwen/qwen35}    \\
Qwen3~\cite{yang2025qwen3}     & Qwen3-14B           & Qwen3-14B-Base         & \url{https://huggingface.co/collections/Qwen/qwen3}     \\
Qwen2.5~\cite{yang2024qwen2}    & Qwen2.5-7B-Instruct & Qwen2.5-7B             & \url{https://huggingface.co/collections/Qwen/qwen25}    \\
Gemma4~\cite{gemma42026modelcard}      & gemma-4-E4B-it      & gemma-4-E4B            & \url{https://huggingface.co/collections/google/gemma-4} \\
Olmo3~\cite{olmo2025olmo}    & Olmo-3-7B-Instruct  & Olmo-3-7B-Instruct-DPO & \url{https://huggingface.co/collections/allenai/olmo-3} \\ \bottomrule
\end{tabular}}
\end{table}
Progress advantage (Proposition~\ref{prop:prog_adv}) is built upon two policies: the RL-trained behavior policy and its reference policy used as a regularization pivot during RL training. Due to limited public resources in terms of the available policy pairs, we consider five representative model families in Table~\ref{tab:policypair} to evaluate the progress advantage in practice. It is common for the industry to just release the final post-trained model to the public while keeping their base and intermediate model checkpoints confidential, but we call on institutions and communities to pay attention and effort to realize a more transparent model development pipeline with fully opened intermediate artifacts, i.e., checkpoints and datasets, as well as step-by-step recipes as pushed by a few leaders~\cite{hall2025marin,olmo2025olmo,blakeman2025nemotron}.

\subsection{Progress $k$-Advantage} \label{sec:apdx:progadv:k}
The default implementation of progress advantage is just using the token probability from $\tilde{\pi}^{*}$ and $\pi_{\rm ref}$ of the realized token at each position in the trajectory. However, some recent works~\cite{tang2025few,shah2025comedy,qi2026rethinking} found that naive token-probability-based regularized RL training sometimes induces nuance instability in gradient estimation; at the same time, confidence-based LLM self-evaluation methods, such as Self-Certainty~\cite{kang2025scalable} and DeepConf~\cite{fu2026deep}, typically adopt probability smoothing over multiple tokens (top-$k$ for instance) to derive stable implementation of confidence rather than the pure (log) probability. Therefore, we additionally propose progress $k$-advantage, a top-$k$ smoothed probability variant of standard progress advantage formulated as below,
\begin{equation} \label{eq:apdx:prog-kadv}
    \text{Progress $k$-Advantage}:=\frac{1}{k}\big(\sum_{i\in \text{Top-}k}\log \tilde{\pi}^{*}(A[i]~|~s)-\sum_{i\in \text{Top-}k}\log \pi_{\rm ref}(A[i]~|~s)\big),
\end{equation} 
\noindent where $\pi(A[i]~|~s)$ denotes the $i$-th token probability value from the output action probability distribution $\pi(A~|~s)$ of the policy $\pi$. In Appendix~\ref{sec:apdx:exp-result}, we compare this top-$k$-smoothed log probability version with the vanilla version.

\section{Details on Experiment Setup} \label{sec:apdx:exp-setup}

\subsection{Baseline Methods} \label{sec:apdx:exp-setup:baseline}
\paragraph{Self-Certainty~\citep{kang2025scalable}} is a training-free, self-confidence-based scoring method proposed for best-of-$N$ selection in reasoning tasks. It evaluates trajectory quality using only the behavior policy's output token probability distribution.
To be specific, at each position $t$, it measures the KL divergence between the model's next-token distribution $\tilde{\pi}^*(\cdot \,|\, s_t)$ and the uniform distribution over the action space $\mathcal{A}$, capturing how peaked the policy is in its predictions.
Equivalently, up to an additive constant, it can be written as a cross-entropy with the uniform distribution.
The trajectory-level score for a trajectory $\tau$ is defined by averaging the per-token self-certainty across all $T$ positions, like below:
\begin{equation} \label{eq:selfcertaintyce}
    \text{Self-Certainty} (\tau) \;=\; -\,\frac{1}{T\,|\mathcal{A}|}\sum_{t=0}^{T-1}\sum_{a \in \mathcal{A}} \log\tilde{\pi}^*(a \,|\, s_t).
\end{equation}
We use a top-20 truncated distribution rather than the full vocabulary distribution for practicality. This provides a scalable signal for best-of-$N$ selection with no additional training, but in agentic settings, it rewards fluent continuations regardless of goal progress, sometimes resulting in poor scoring on correct but low-frequency tool-call strings.

\paragraph{DeepConf~\citep{fu2026deep}} is also a training-free, confidence-based scoring method designed and validated mainly on math or STEM domain reasoning tasks.
It refines confidence estimation by introducing local confidence measures over sliding token windows, motivated by the observation that global trace-level averages can dilute some important finer signals per step.
At each position $t$, the per-token confidence is defined as the average log-probability of the top-$k$ tokens under $\tilde{\pi}^*(\cdot \,|\, s_t)$ as
$C_t = \frac{1}{k}\sum_{a \in \mathrm{Top}\text{-}k(\tilde{\pi}^*(\cdot \,|\, s_t))} \log \tilde{\pi}^*(a \,|\, s_t).$
The group confidence is then defined over by sliding window $G_i=\{i-w,i-w+1,...,i\}$ with $w$ previous tokens, such as,  $C_{G_i}=\frac{1}{|G_i|}\sum_{t \in G_i} C_t$. 
While DeepConf authors adopted an overlapping sliding window to define group confidence, this is not valid for agent rollouts that mix the agent's own output tokens with the environment-side observation tokens. Therefore, we define the group confidence as the average of token confidence per agent's action step without overlapping neighbor steps. Finally, to define a trajectory-level score for the TTS and UQ scenarios, we adopt two instantiations of DeepConf as follows.

(1) \textit{DeepConf Tail} as the last step group confidence, capturing action quality at the termination phase:
\begin{equation}
    \text{DeepConf}_{\text{Tail}}(\tau) \;=\; C_{G_{\text{last}}}.
\end{equation}
(2) \textit{DeepConf B10} averages the bottom-10\% of group confidences across the trajectory, focusing on the least-confident segments. Letting $\mathcal{B}_{10}$ denote the index set of groups with the lowest 10\% of $\{C_{G_i}\}$, which is equal to the lowest 10\% per-step confidence set,
\begin{equation}
    \text{DeepConf}_{\text{B10}}(\tau) \;=\; \frac{1}{|\mathcal{B}_{10}|}\sum_{i \in \mathcal{B}_{10}} C_{G_i}.
\end{equation}
Both instantiations rely solely on the behavior policy probability $\tilde{\pi}^*$, sharing the same limitation as Self-Certainty in agentic settings: high local confidence from behavior policy alone may entangle generic linguistic priors with goal-directed progress.

\paragraph{LLM-as-a-Judge~\cite{zheng2023judging,chiang2023can,bubeck2023sparks}.} LLMs can be prompted as a judge to assign bounded scalar score assessments to outputs produced by other LLMs. This LLM-as-a-Judge paradigm has become a scalable substitute for human preference annotation in open-ended evaluation~\cite{liu2023geval,zheng2023judging,dubois2024lengthcontrolled} and is increasingly used as a feedback source in model-development pipelines, including reward modeling and self-alignment~\cite{lee2024rlaif,yuan2024self,son2024llmjudge}. More recently, judge models have also shown promise for evaluating goal-oriented, long-horizon agent trajectories, where the evaluator must reason over the full action-observation history rather than only the final response~\cite{zhuge2025agent,xi2025agentprm}. In our UQ experiment\footnote{Since the cost of API usage burdens a lot in the best-of-N sampling scenario, we confined these baselines to the UQ setup.}, we use Claude-Sonnet-4.6~\cite{anthropic2026sonnet46}, a strong and relatively cost-effective modern LLM judge, to predict whether an agent succeeds in achieving the task goal using the prompt depicted in Figure~\ref{fig:uq-judge-prompt}.

\paragraph{WildReward~\cite{peng2026wild}.}
We adopt \texttt{WildReward-8B}~\cite{peng2026wild} as an off-the-shelf pre-trained outcome reward model\footnote{\url{https://huggingface.co/THU-KEG/WildReward-8B}}, which is trained over massive user-chatbot interaction data, using it as a drop-in scorer across TTS and UQ for trajectory-level scoring. Given a trajectory input, it returns a single scalar reward estimate $\hat{r}:=1+\sum_j
\sigma(z_j)\in[1,5]$ via CORAL ordinal regression~\cite{cao2020rank} where $z_j$ for $j=1,2,3,4$ denotes individual logit values from 4-way threshold classification head. Figure~\ref{fig:wildreward-prompt} depicts the prompt used for all the downstream evaluation.

\paragraph{ThinkPRM~\cite{khalifa2025process}} is a PRM trained on mathematical reasoning data\footnote{\url{https://huggingface.co/launch/ThinkPRM-7B}}, that provides a step-by-step verification on the model-generation trajectory with per-step verdicts on correct-or-incorrect. We extract the per-step $P(\text{correct})$ from the vLLM inference logprobs and average across steps to derive trajectory-level reward estimate. We adapted the official prompt template tailored to the per-task setting: the template in Figure~\ref{fig:thinkprm-uqtts-prompt} for UQ and TTS and the template in Figure~\ref{fig:thinkprm-fa-prompt} for failure attribution.

\paragraph{AgentPRM~\cite{xi2025agentprm}.} As a task-specific training-based PRM baseline, we reproduce AgentPRM, which trains a reward head $h_{\phi}$ over the policy backbone to jointly predict each step's promise and progress under a combined regression objective.
We follow the original recipe of the authors with \texttt{Qwen2.5-7B-Instruct} as both the behavior policy and the PRM backbone: we first SFT the behavior policy with LoRA ($r=64$ and $\alpha=128$) on 300 AgentTraj-L expert trajectories drawn from the AgentGym WebShop train split\footnote{For the experiment in Tab.~\ref{tab:tts-training-based-rm}, we used AgentGym wrapper while used WebShop codebase for other TTS experiments.} (with the 100 test samples held out). Then, we continually LoRA fine-tune the backbone with a randomly initialized linear reward head, $h_{\phi}:\mathbb{R}^{3584}\rightarrow\mathbb{R}$, for 3 epochs at 1e-5 learning rate on 8000 on-policy trajectories sampled with temperature 1.0, with GAE parameter set $\lambda = 0.95$ and $\gamma = 1$.

\subsection{Application-Specific Details} \label{sec:apdx:exp-setup:detail}
\subsubsection{Test-time scaling (TTS)} \label{sec:apdx:exp-setup:detail:tts}
\paragraph{Problem definition.}
We consider best-of-N (parallel) sampling~\cite{lightman2023let,brown2024large} as a testbed for reward models on a test-time compute scaling application, where we sample multiple responses (with non-zero generation temperature parameter) from an LLM given the same task prompt in parallel and score each trajectory to select the one that results in the \textit{best} score among them for the actual evaluation. That is, given a trajectory $\tau$ starts from a prompt $s$ and a task-specific (typically binary) success-failure evaluator $y(\tau~|~s)\in\{0,1\}$, we measure average task success rate on a dataset $\mathcal{D}$ of task prompts as below,
\begin{align} \label{eq:tts-bon}
    \frac{1}{|\mathcal{D}|}\sum_{s\in\mathcal{D}}
    y( \tilde{\tau}~|~s)~~~&\text{for}~~~\tilde{\tau}=\arg\max_{\tau\in\mathcal{T}} r(s,\tau)~~~\text{with}~~~\mathcal{T}=\{M^{(i)}_{\pi}(\cdot|s,c)\}_{i=1}^{N},
\end{align}
\noindent where $M^{(i)}_{\pi}(\cdot|s,c)$ denotes the $i$-th trajectory generated by an LLM $M_{\pi}$ with a policy $\pi$ given a temperature parameter $c \geq 0.0$ and $r(\cdot)$ denotes a reward score computed over the trajectory $\tau$. The aim of this evaluation is to compare different reward models $r(\cdot)$ that yield the best success rate. 

\paragraph{Benchmark.} 
We adopt four different benchmarks that are specifically designed for evaluating LLMs' agentic capability: BFCLv4~\cite{patil2025berkeley}, WebShop~\cite{yao2022webshop}, AgentDojo~\cite{debenedetti2024agentdojo}, and $\tau^{2}$-bench Airline~\cite{yao2024tau,barres2025tau}. Below we elaborate on each benchmark, while deferring the description of the $\tau^2$-bench to the next subsection~\ref{sec:apdx:exp-setup:detail:uq}. Table~\ref{tab:tts-benchmarks} summarize the statistics
\begin{table}[th]
\vspace{-0.7em}
\caption{\textbf{Test-time scaling benchmarks.} For each benchmark, we sample $N=8$ rollouts per task at temperature $c$, then score every rollout with each reward scoring method and select the argmax. AgentDojo aggregates four suites; the per-suite task counts are listed in parentheses.}
\label{tab:tts-benchmarks}
\centering
\small
\begin{tabular}{llcc}
\toprule
Benchmark & Split & \# tasks & $c$ \\
\midrule
BFCLv4-MT~\cite{patil2025berkeley} & Multi-turn base split & 200 & 0.4 \\
WebShop~\cite{yao2022webshop} & First 100 tasks subset with \texttt{items\_shuffle\_1000.json} & 100 & 0.7 \\
AgentDojo~\cite{debenedetti2024agentdojo} & Workspace (40), Slack (21), Banking (16), and Travel (20) & 97  & 0.4 \\
$\tau^{2}$-bench~\cite{yao2024tau,barres2025tau} & Airline & 50  & 0.7 \\
\bottomrule
\end{tabular}
\vspace{-0.7em}
\end{table}

\textbf{BFCLv4} is a benchmark focusing on assessing an agent's function using capability. Here, the agent must emit a sequence of tool calls whose arguments and ordering exactly match a reference, with cross-turn argument propagation graded by a deterministic verifier. We use the \texttt{multi\_turn\_base} category (200 tasks) for the format-fidelity probing test while ignoring the harder long and missing-info categories.

\textbf{AgentDojo} is a benchmark of computer-use task suites with realistic mock APIs (reading email, posting to Slack, banking transactions, travel booking) where the agent must complete a free-form natural-language request by issuing tool calls. We run four suites, Workspace (40), Slack (21), Banking (16), Travel (20), totaling 97 tasks, where the success requires both the correct final state and a clean termination.

\textbf{WebShop} is a simulated e-commerce environment in which the agent navigates HTML-rendered product search and click pages to fulfill a shopping instruction conveyed through natural language. The reward is shaped in $[0,1]$ based on attribute overlap with the target product. We evaluate on the first $100$ tasks with 1000 products (\texttt{items\_shuffle\_1000.json} split) and 30 per-rollout interaction step limits, which is the standard WebShop split used in prior work.

\paragraph{Scenario-specific baseline.} We consider three baselines in this best-of-N setup: greedy decoding, mean-of-N, and pass@N. Since the default inference mode of TTS is generation with non-zero temperature, we consider a zero-temperature greedy decoding as a reference for each benchmark. This can be useful to check whether the non-zero temperature exploratory decoding is more helpful than a deterministic, exploitative decoding strategy on each task~\cite{song2025good}. Meanwhile, mean-of-N is literally the average success rate of N candidate trajectories, representing the average competency of exploratory decoding. Finally, we also report pass@N as an oracle, upper bound score achievable by any selection method, indicating that the rate of at least one trajectory among N candidates succeeds.

\subsubsection{Uncertainty quantification (UQ)} \label{sec:apdx:exp-setup:detail:uq}
\paragraph{Problem definition.} 
UQ is a representative application in the deployment-phase monitoring of LLM agents~\cite{hu2024uncertainty,zhang2026auq,oh2026uncertainty}, which is becoming a central interest of research to realize trustworthy AI in the wild.
Similar to the TTS problem setup, given the task prompt $s$, a trajectory-level reward score $r(s,\tau)$ predicts whether the trajectory $\tau$ will end in success as below,
$\mathbb{I}\big( r(s,\tau) > H \big)$ given a classification threshold $H\geq 0$. By following the evaluation standard in LLM UQ research~\cite{malinin2021uncertainty,kuhn2023semantic}, we report the area under the receiver operating characteristic curve (AUROC) which serves as a threshold-independent, balanced measure of binary prediction quality.

\paragraph{Benchmark.} 
We adopt $\tau^{2}$-bench~\cite{yao2024tau,barres2025tau} Airline and Retail domains as our main testbed for UQ (Telecom domain was excluded from consideration since the modern agents' performances are saturated on that domain). Each domain has a detailed policy prompt that lists domain-specific constraints and ground rules, and the agent LLM equips that policy as a system prompt to define its domain-specific background context. In the meantime, the benchmark also hosts another LLM as a user simulator (we adopt Kimi-K2.5~\cite{kimi2026} given its outstanding cost-effectiveness), who has a brief persona with synthetic personal information and also has a seed prompt to define their goal per task. Under this scaffolding, the agent spans a long-horizon trajectory by interacting with the user and pre-defined tools to achieve the complex goal\footnote{Some example trajectories can be explored in this \href{https://taubench.com}{official online leaderboard}.}. Different to the TTS setup, we here fix the temperature parameters for both agent and user LLMs to zero to make the trajectory generation deterministic, i.e., greedy decoding mode, and quantify uncertainty once over these deterministic generation passes to prevent any unexpected confounding effects. The greedy decoding performance per-model in these two domains is provided in Table~\ref{tab:taugreedy}, where we found that the latest two model backbones \texttt{Gemma4-4B} and \texttt{Qwen3.5-9B} show balanced moderate performance on both domains.
\begin{table}[!ht]
    \centering
    \vspace{-1em}
    \caption{\textbf{$\tau^{2}$-bench Airline and Retail greedy decoding success rate}. N denotes \# of samples.}
    \label{tab:taugreedy}
    \small
    \begin{tabular}{l|cccc}
    \toprule
        Domain & \texttt{Gemma4-4B} & \texttt{Qwen3.5-9B} & \texttt{Qwen3-14B} & \texttt{Olmo3-7B} \\ \midrule
        Airline ($N=50$) & 34.0  & 60.0 & 12.0 & 30.8 \\ 
        Retail ($N=114$) & 45.6 & 64.9 & 50.0 & 16.7 \\ \bottomrule
    \end{tabular} 
    \vspace{-1.2em}
\end{table}

\paragraph{Scenario-specific baseline.}
We consider the LLM-as-a-Judge baseline, which prompts a powerful LLM to predict the binary outcome of success or failure given a realized agent trajectory. We adopt Claude-Sonnet-4.6 as our LLM judge, given its remarkable performance while requiring way cheaper API cost than its competitors. See the section~\ref{sec:apdx:exp-setup:baseline} for more description.

\subsubsection{Failure attribution (FA)} \label{sec:apdx:exp-setup:detail:fa}
\paragraph{Problem definition.}
We pose an agentic system that consists of multiple LLM agents, each equipped with tool-calling and inter-agent communication capabilities, collectively designed to solve a complex, long-horizon task. A trajectory $\tau\in\mathcal{T}$ generated from this agentic system is simply defined as $\tau=(s,a_1,a_2,...,a_{T(\tau)})$ with the user's initial prompt $s$ and sequence of agents' actions $a_t$ with varying per-trajectory length $T(\tau)$.
Given a failure trajectory $\tau$ and a ground truth error step annotator $y(\tau)=\{t~:~a_t~\text{is an incorrect action contributing to the system failure.}\}$, the failure attribution task aims to predict a \textit{decisive error step}, $t_{\rm err}:=\min_{t} y(\tau) $ that becomes \textbf{the earliest critical error causing the failure of the system}.
For all methods, we predict the decisive error step as the index of the steepest cummulative reward drop step $\hat{t}_{\rm err}:=\arg\min_t \sum_{i=0}^{t} r(s_i,a_i) - \sum_{i=0}^{t-1} r(s_{i-1},a_{i-1})$ (which is equal to the minimum step reward index, $\arg\min_t r(s_t,a_t)$). Then, we measure the step-level prediction accuracy $\mathbb{I}(t_{\rm err}=\hat{t}_{\rm err})$ over the dataset.

\paragraph{Benchmark.}
We adopt Who\&When~\cite{zhang25which} benchmark where the seed tasks are drawn from two sources: (1) GAIA~\cite{mialon2024gaia}, which contains queries requiring information processing from multiple modalities, such as PDFs, spreadsheets, images, videos, and audio, as well as web browsing and coding capabilities; and (2) AssistantBench~\cite{yoran-etal-2024-assistantbench}, which asks agents to play with multiple websites across various topics in geography, visual arts, biology, and so on. From the seed task prompt, trajectories were then generated by two agentic systems, CaptainAgent~\cite{song2025adaptive} and Magnetic-One~\cite{fourney2024magenticone}, using GPT-4o as a behavior policy. The dataset consists entirely of failure trajectories, comprising 184 trajectories in total, each annotated with a decisive error step label. 

\paragraph{Scenario-specific baseline.}
We consider AgenTracer~\cite{zhang2025agentracer} as a training-based method specifically trained on failure attribution tasks. We report the performance of Zhang et al.~\cite{zhang2025agentracer} which conducts GRPO training from \texttt{Qwen3-8B} base model on 2.5K failure trajectories with step-level annotations. Given an arbitrary-length trajectory, AgenTracer predicts a decisive error step with a reasoning trace.

\subsection{Inference Setup and Configuration} \label{sec:apdx:exp-setup:configuration}
The aforementioned three classes of experiments, TTS, UQ, and FA, share the following common configuration unless noted. All LLMs are loaded in \texttt{bfloat16} on a single GPU with model-family-tailored tokenizers and chat templates (e.g. \texttt{qwen3\_coder} \& \texttt{qwen3} reasoning parser for Qwen3.5; \texttt{gemma4} for Gemma4; \texttt{olmo3} for Olmo3).

In \textbf{TTS} setup, we use vLLM~\cite{kwon2023efficient} to generate trajectory with \texttt{max\_model\_len}=32768 and set \texttt{enable\_thinking}=\texttt{false} for all model backbones (except the Olmo3) that support inference mode selection. Greedy decoding baselines adopt zero generation temperature and \texttt{top\_p}=1.0; the best-of-$N$ stochastic trajectory generation set \texttt{top\_p}=0.95, \texttt{max\_new\_tokens}=1024, and temperature as 0.7 on $\tau^2$-bench and WebShop, whereas BFCL and AgentDojo adopt 0.4 temperature. The user simulator in $\tau^2$-bench is held at zero-temperature, so the inter-trial volatility comes from the agent only.

In \textbf{UQ} and \textbf{FA}, we score pre-generated greedy decoding trajectories by re-tokenizing the full conversation log and running a single forward pass per trajectory through the standard HuggingFace transformers library as the engine. We left-truncate to a context window of \texttt{max\_length}=16384 tokens and read the per-position log-softmax for confidence-based and our method.

\section{Additional Results} \label{sec:apdx:exp-result}
\paragraph{Is top-$k$ averaging over log token probability helpful?} We first compare our default progress advantage, $\log\tilde{\pi}^{*}(a|s)-\log\pi_{\rm ref}(a|s)$, with a top-$k$ smoothed version, progress $k$-advantage. Table~\ref{tab:tts-bon-main-apdx} and Table~\ref{tab:uq-on-main-apdx} present the results in TTS and UQ, respectively, and Figure~\ref{fig:fa-main-apdx} provides results on the FA.

\begin{minipage}{\textwidth}
\begin{minipage}[ht]{0.48\textwidth}
\captionof{table}{\textbf{TTS through best-of-8 sampling}. We compare reward scoring methods across two LLM backbones, reporting the success rate (\%) of the selected trajectory averaged across four datasets (BFCLv4, WebShop, AgentDojo, and $\tau^2$-Airline).}
\label{tab:tts-bon-main-apdx}
\centering
\resizebox{1.0\textwidth}{!}{%
\begin{tabular}{lcc}
\toprule
Scoring Method     & \texttt{Gemma4-4B} Avg.  & \texttt{Qwen3.5-9B} Avg. \\ \midrule
\rowcolor{mygray}
Pass@N (oracle)    &  45.4 & 67.5 \\
\rowcolor{mygray}
Greedy Decoding     & 33.4 & 54.6 \\
Mean-of-N &  33.1 & 54.7 \\ \midrule
WildReward-8B~\cite{peng2026wild} &  33.1 & 54.8 \\
ThinkPRM-7B~\cite{khalifa2025process}  & 30.9 & 52.2 \\
ThinkPRM-14B~\cite{khalifa2025process} & 33.6 & 54.9 \\ \midrule
Self-Certainty~\cite{kang2025scalable}  & 29.0 & 51.5 \\
DeepConf Tail~\cite{fu2026deep}  & 31.4 & 53.2 \\
DeepConf B10~\cite{fu2026deep} & 27.4 & 55.8 \\ \midrule
\rowcolor{lightpurple}
Progress $k$-Advantage &  \underline{34.1} & \underline{58.1} \\ 
\rowcolor{lightpurple}
Progress Advantage & \textbf{38.8} & \textbf{62.1} \\
\bottomrule
\end{tabular}
}
\end{minipage}
\hfill
\begin{minipage}[ht]{0.48\textwidth}
\captionof{table}{\textbf{UQ for trajectory monitoring}. We compare scoring methods across four LLM backbones on $\tau^2$-bench Airline and Retail to predict an agent's success on each model's greedy-decoding trajectory. AUROC averaged across the four backbones (\texttt{Gemma4-4B}, \texttt{Qwen3.5-9B}, \texttt{Qwen3-14B}, and \texttt{Olmo3-7B}).
}
\label{tab:uq-on-main-apdx}
\centering
\resizebox{\textwidth}{!}{%
\begin{tabular}{lcc}
\toprule
Scoring Method     & $\tau^{2}$-Airline Avg. & $\tau^{2}$-Retail Avg. \\ \midrule
\rowcolor{mygray}{\claudelogo} Sonnet-4.6~\cite{anthropic2026sonnet46}             & 0.644 & 0.818 \\
WildReward-8B~\cite{peng2026wild}              & 0.420 & 0.596 \\
ThinkPRM-7B~\cite{khalifa2025process}                & 0.457 & 0.558 \\
ThinkPRM-14B~\cite{khalifa2025process}               & 0.520 & 0.591 \\ \midrule
Self-Certainty~\cite{kang2025scalable}            & 0.658 & 0.441 \\
DeepConf Tail~\cite{fu2026deep}             & 0.581 & 0.429 \\
DeepConf B10~\cite{fu2026deep}              & 0.669 & 0.446 \\ \midrule
\rowcolor{lightpurple}
Progress $k$-Advantage     & \underline{0.732} & \underline{0.646} \\ 
\rowcolor{lightpurple}
Progress Advantage        & \textbf{0.781} & \textbf{0.671} \\
\bottomrule
\end{tabular}
}
\end{minipage}
\end{minipage}

We observe that both the progress advantage and $k$-smoothing variant greatly outperform the baseline methods on both TTS and UQ scenarios by consistently achieving the best and the second-best performances. In these two scenarios, the default progress advantage beats the $k$-smoothing version, showing its effectiveness grounded in theoretical derivation. Meanwhile, progress $k$-advantage exceeds the default progress advantage on some of the FA setups with \texttt{Gemma4-4B} backbone, wherein the precise per-step credit assignment is crucial, suggesting the smoothed representation of probability may sometimes be better than the exact one to model step-level progress. This result implies that one can carefully choose the representation of token probability (pure token probability, $k$-averaged, etc.) for the intended use and characteristics in a downstream task and data.
\begin{figure}[t!]
    \centering
    \includegraphics[width=0.5\linewidth]{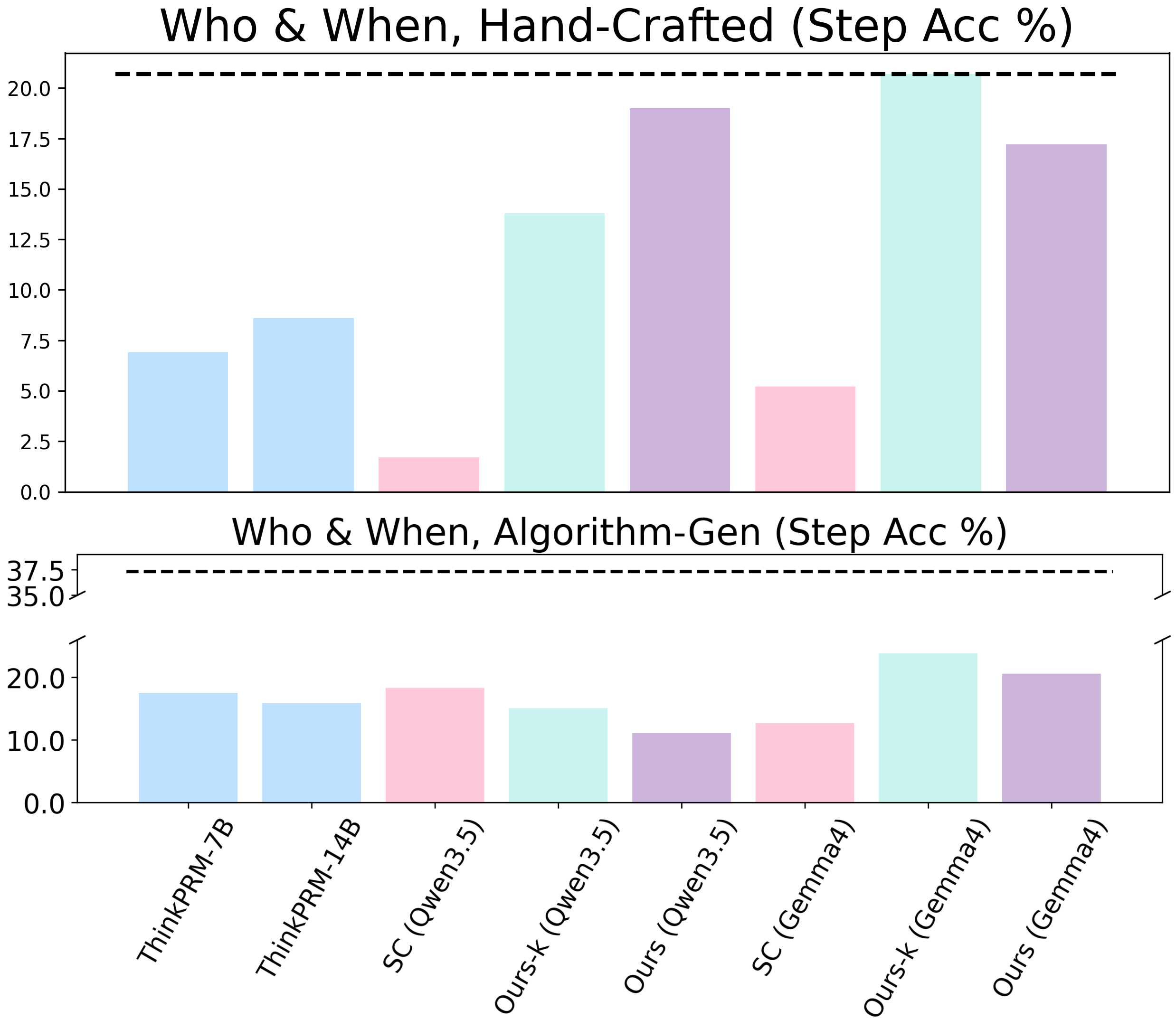}
    \vspace{-0.4em}
    \caption{\textbf{Who \& When step-level accuracy}. 
    We predict when the agent system makes a decisive error. SC denotes Self-Certainty~\cite{kang2025scalable}, Ours-k denotes the progress $k$-advantage, and the dashed line denotes AgenTracer~\cite{zhang2025agentracer} specifically trained on this failure attribution task through RL-training.
    }
    \vspace{-0.7em}
    \label{fig:fa-main-apdx}
\end{figure}

\paragraph{Comparison between progress advantage and its ingredients.} Since our progress advantage is constructed with the log probability of the behavior policy $\log\tilde{\pi}^{*}(a|s)$ and the reference policy $\log\pi_{\rm ref}(a|s)$, the natural question is whether we should blend the two log probabilities rather than using one of them for simplicity and efficiency. In Figure~\ref{fig:qual_anal}, we showed how the pure policy log probability $\log\tilde{\pi}^{*}(a|s)$ results in poor scoring, while progress advantage derives desirable scoring through a qualitative analysis. In this paragraph, we further provide a summary of quantitative results on the comparison between the progress advantage $\log\frac{\tilde{\pi}^{*}(a|s)}{\pi_{\rm ref}(a|s)}$, behavior policy log probability $\log\tilde{\pi}^{*}(a|s)$, and reference policy log probability $\log\pi_{\rm ref}(a|s)$ across eight UQ scenarios. Table~\ref{tab:component_ablation} shows the average ranking and corresponding average AUROC. We see that the progress advantage outperforms its two ingredients by large margins.
\begin{table}[hb!]
\caption{\textbf{Comparison between the progress advantage and its ingredients.} On the eight UQ scenarios (two domains and four model backbones), we compute the average ranking and AUROC of each reward signal across all possible combinations of token and step aggregations (25 in total) per method. We see that progress advantage consistently outperforms the individual ingredients, such as the log probability of the behavior policy and reference policy.}
\label{tab:component_ablation}
\centering
%\resizebox{\linewidth}{!}{
\begin{tabular}{l|cc}
\toprule
Scoring Method & Avg. Rank by Best AUROC & Avg. Best AUROC \\ \midrule
\rowcolor{lightpurple}
$\log\frac{\tilde{\pi}^{*}(a|s)}{\pi_{\rm ref}(a|s)}$ & \meanstd{\textbf{1.44}}{0.62} & \meanstd{\textbf{0.732}}{0.059} \\
$\log\tilde{\pi}^{*}(a|s)$   & \meanstd{2.25}{0.89} & \meanstd{0.679}{0.044} \\
$\log\pi_{\rm ref}(a|s)$ & \meanstd{2.31}{0.70} & \meanstd{0.695}{0.043} \\ \bottomrule
\end{tabular}
%}
\end{table}

\begin{figure*}[thb]
    \centering
    \makebox[\textwidth][c]{
    \includegraphics[width=1.025\linewidth]{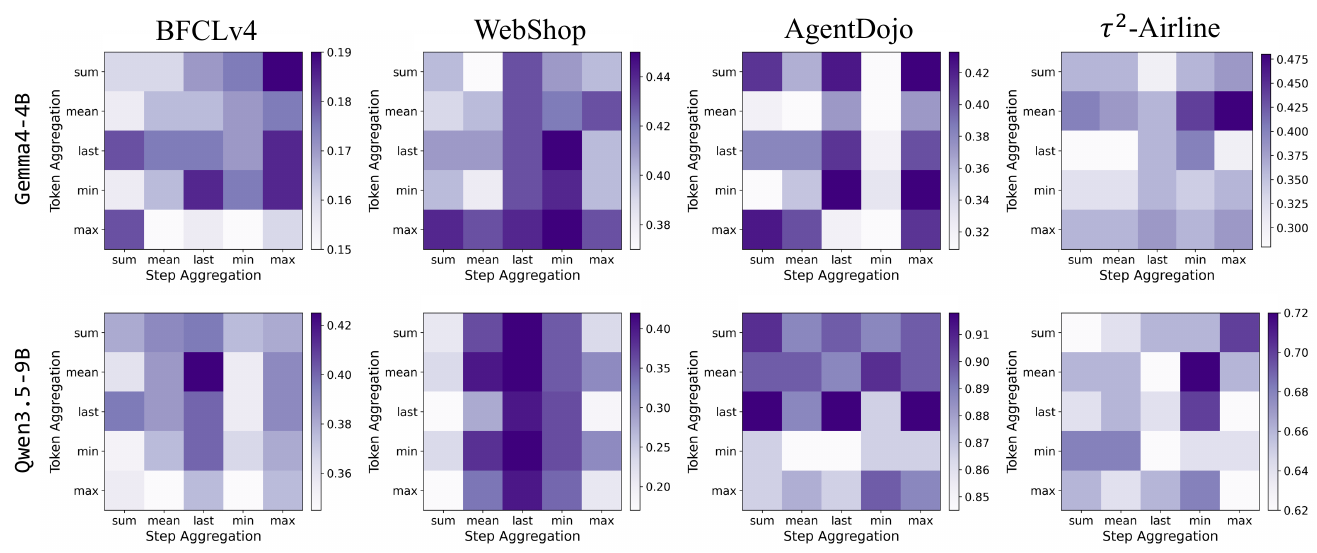}}
    \caption{\textbf{Varying combinations of token and step aggregation strategy for progress advantage in best-of-N.} We sweep 25 combinations of token-wise and step-wise aggregation of progress advantage over four datasets and two model backbones in the best-of-8 scenario.}
    \label{fig:agg_heatmap_full_tts}
    \vspace{-1em}
\end{figure*}
\paragraph{Extended results on advantage aggregation strategies.} As mentioned in Section~\ref{sec:method:practical} and Section~\ref{sec:exp:anal}, aggregation of the advantages across tokens and steps significantly affects the downstream utility of progress advantage. For the best-of-N scenarios (Figure~\ref{fig:agg_heatmap_full_tts}), the best working aggregation strategy is not universal across datasets or models. Some datasets, such as WebShop and AgentDojo, show somewhat robust results across different aggregation methods. For example, choosing \textsc{max} operator as a token aggregation and \textsc{last} or \textsc{min} as step aggregation yields promising results for \texttt{Gemma4-4B} on WebShop, whereas \textsc{mean} or \textsc{min} as token aggregation and \textsc{last} as step aggregation become the winning tickets for \texttt{Qwen3.5-9B}; other datasets exhibit sensitivity depending on the aggregation where (\textsc{mean}, \textsc{last}) as the (token, step) aggregation pair stands for the sole winner in BFCLv4 dataset with \texttt{Qwen3.5-9B}.

In the meantime, the trend in UQ (Figure~\ref{fig:agg_heatmap_full_uq}) is better interpretable, where we see (\textsc{max}, \textsc{mean}) combination is a winning strategy for Airline, whereas (\textsc{min}, \textsc{last}) combination is a winner for Retail. We can conclude that extreme tokens that produce maximum or minimum progress advantage are informative to define a per-step signal in UQ. Meanwhile, the effectiveness of step-wise aggregation reflects the domain structures: most of \textbf{tasks in the retail domain} (\textit{return}, \textit{exchanges}, \textit{order modification}) have some clear terminating action, and a single tool call can commit the outcome sometimes---incentives \textsc{last} as a suitable operation. Meanwhile, most of \textbf{tasks in the airline domain} (\textit{cancellations}, \textit{baggage policy}, \textit{membership status checks}) are kind of policy negotiations where the agent and user talk through the decision over multiple steps, and the success/failure is a slow event distributed across the whole conversation, making \textsc{mean} as a go-to operation.
\begin{figure*}[htb]
    \centering
    \makebox[\textwidth][c]{
    \includegraphics[width=1.025\linewidth]{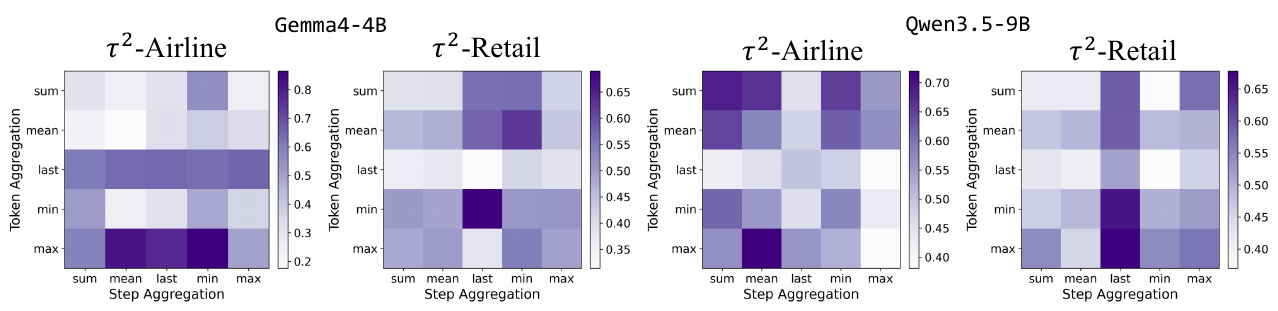}}
    \caption{\textbf{Varying combinations of token and step aggregation strategy for progress advantage in UQ.} We sweep 25 combinations of token-wise and step-wise aggregation of progress advantage over two domains in $\tau^2$-bench and two model backbones in the uncertainty quantification scenario.}
    \label{fig:agg_heatmap_full_uq}
    \vspace{-1.2em}
\end{figure*}

\begin{table}[bht]%{0.5\textwidth}
%\vspace{-1.2em}
\centering
\caption{\textbf{Average rank (lower is better) with per-setting optimal vs.\ fixed aggregation}. Over the eight settings of Tab.~\ref{tab:tts-bon-main} (TTS) and Tab.~\ref{tab:uq-on-main} (UQ), we compare \emph{Per-setting} which uses the best (token, step) aggregation for each setting, and \emph{Fixed} that uses (\textsc{last}, \textsc{last}) for all TTS settings and (\textsc{max}, \textsc{mean}) for all UQ settings.}
\label{tab:agg-fixed-rank}
%\vspace{-0.6em}
%\resizebox{0.5\textwidth}{!}{%
\begin{tabular}{l|cc|cc}
\toprule
\multirow{2}{*}{Scoring Method} & \multicolumn{2}{c|}{TTS} & \multicolumn{2}{c}{UQ} \\
 & Per-setting & Fixed & Per-setting & Fixed \\ \midrule
\rowcolor{mygray}{\claudelogo} Sonnet-4.6 judge & -- & -- & \underline{2.25} & \textbf{2.00} \\
Mean-of-N        & 3.62 & 3.44 & -- & -- \\ \midrule
WildReward-8B    & 4.00 & 3.81 & 5.50 & 5.00 \\
ThinkPRM-7B      & 5.62 & 5.44 & 5.50 & 5.38 \\
ThinkPRM-14B     & \underline{3.56} & \underline{3.38} & 4.62 & 4.12 \\ \midrule
Self-Certainty   & 6.81 & 6.75 & 5.25 & 5.00 \\
DeepConf Tail    & 5.19 & 5.12 & 6.00 & 5.88 \\
DeepConf B10     & 5.69 & 5.62 & 5.25 & 4.88 \\ \midrule
\rowcolor{lightpurple}
Progress Advantage & \textbf{1.50} & \textbf{2.44} & \textbf{1.62} & \underline{3.75} \\ \bottomrule
\end{tabular}
%\vspace{-1em}
\end{table}

\paragraph{Progress Advantage without per-setup aggregation tuning.}
In Tab.~\ref{tab:tts-bon-main} and Tab.~\ref{tab:uq-on-main}, we report progress advantage with the best (token, step) aggregation for each setting, which is optimistic; e.g., they show how much task-relevant information the signal carries when it is aggregated at the best, rather than how well it works out of the box.
To provide a recipe that can be used without tuning, we also fix one aggregation per application and apply it to every setting: (\textsc{mean}, \textsc{max}) for TTS and (\textsc{max}, \textsc{mean}) for UQ. We pick each of them as the combination with the best average rank among the 25 combinations in Fig.~\ref{fig:agg_heatmap_full_tts} and Fig.~\ref{fig:agg_heatmap_full_uq}.
As Tab.~\ref{tab:agg-fixed-rank} shows, dropping per-setting selection slightly worsens the average rank of progress advantage. However, even with a single fixed aggregation without any task-specific tuning, progress advantage still ranks \textbf{first of eight methods on TTS} and \textbf{second of eight on UQ}, where only the proprietary Sonnet-4.6 judge ranks higher, implying its practical usefulness in the wild.

\paragraph{Visualization of progress advantage evolution.} We have observed promising results of progress advantage in the UQ setup so far. To dive deeper into this success, we analyze how the progress advantage actually evolves step-by-step across the whole trajectory. In Figure~\ref{fig:advantage_evo}, we visualize group average per-step progress advantage for success and failure trajectory groups, where we normalize the step index to the [0,1] range since each trajectory has a different length. In the Airline domain, we see that progress advantage clearly separates the two groups from the very beginning. In the Retail domain, on the other hand, it shows a tied trend most of the time, but gives a higher advantage to the success group at the very end, implying that while the progress advantage shows its effectiveness broadly, the way it works can vary across domains.
\begin{figure}[!h]
    \centering
    \includegraphics[width=\linewidth]{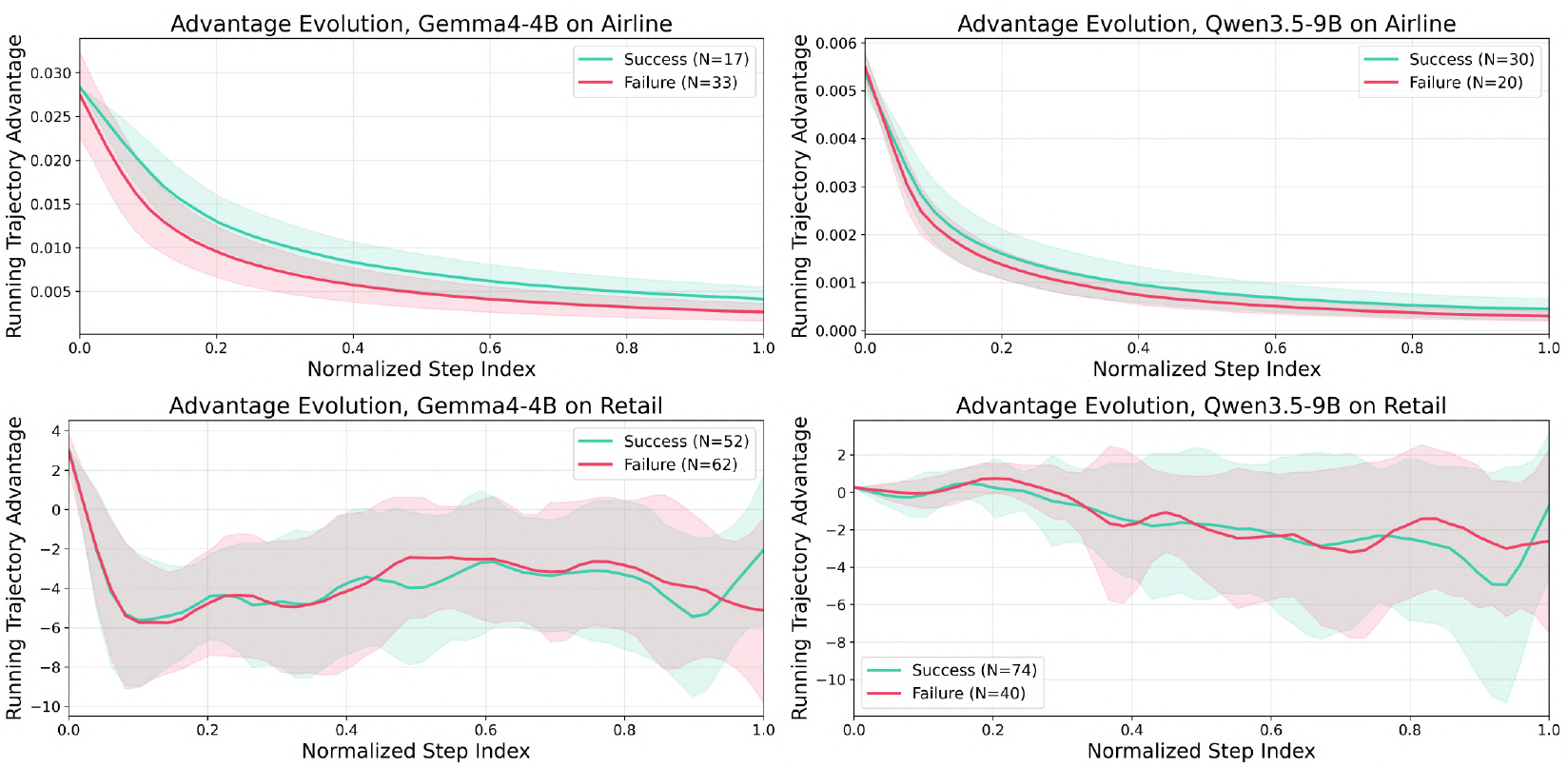}
    \caption{\textbf{Progress advantage evolution across trajectory}. We visualize group average per-step progress advantage over the $\tau^{2}$-bench greedy decoding trajectories generated by \texttt{Gemma4-4B} and \texttt{Qwen3.5-9B} where we apply \textsc{max} and \textsc{min} aggregation across tokens within each step for Airline and Retail domains, respectively, and apply \textsc{mean} and \textsc{last} aggregation across steps for Airline and Retail domains to get the running advantage. The group is defined by binary success and failure outcomes. The shade denotes one standard deviation across within-group per-step progress advantage.}
    \label{fig:advantage_evo}
\end{figure}

\begin{table}[t!]
\caption{\textbf{Test-time scaling through best-of-8 sampling on WebShop.} Given eight sample trajectories from \texttt{Qwen2.5-7B-Instruct} as a behavior policy with temperature 0.7, we compare training-based reward models as well as training-free confidence-based methods to the default progress advantage and its GRPO task-specific fine-tuned variant.}
\label{tab:tts-training-based-rm}
\centering
\begin{tabular}{lcc}
\toprule
Method                      & Training & Success Rate \\ \midrule
\rowcolor{mygray}
Pass\@N (oracle)            & \xmark        & 45.0          \\
\rowcolor{mygray}
Greedy Decoding             & \xmark        & 30.0         \\ 
Mean-of-N                   & \xmark        & 31.0         \\ \midrule
WildReward-8B~\cite{peng2026wild} & \cmark        & 32.0          \\
ThinkPRM-7B~\cite{khalifa2025process} & \cmark        & 32.0          \\
ThinkPRM-14B~\cite{khalifa2025process} & \cmark        & \underline{35.0}           \\
AgentPRM-7B~\cite{xi2025agentprm} & \cmark        & 33.0         \\ \midrule
Self-Certainty~\cite{kang2025scalable} & \xmark        & 30.0          \\
DeepConf Tail~\cite{fu2026deep} & \xmark        & 29.0    \\
DeepConf B10~\cite{fu2026deep} & \xmark        & 27.0         \\ \midrule
\rowcolor{lightpurple}
Progress Advantage          & \xmark        & \underline{35.0}           \\
\rowcolor{lightpurple}
Progress Advantage w/ GRPO  & \cmark        & \textbf{38.0}        \\ \bottomrule
\end{tabular}
\end{table}
\paragraph{Comparison with AgentPRM.} Due to the resource constraint, we mainly considered WildReward~\cite{peng2026wild} and ThinkPRM~\cite{khalifa2025process} in our main paper as training-based PRM baselines, which are trained on general multi-turn interaction or reasoning datasets rather than the actual downstream task datasets. In this paragraph, we provide results with an additional PRM baseline, AgentPRM~\cite{xi2025agentprm}, \textbf{which is directly trained on a specific downstream task}, e.g., WebShop, with a value head to predict the Q-value per step. Specifically, we follow the recipe in AgentPRM by first performing a small-scale SFT from \texttt{Qwen2.5-7B-Instruct} and then going through the actual reward modeling with step-wise advantage estimations. In Table~\ref{tab:tts-training-based-rm}, although AgentPRM beats the comparable-size general pre-trained reward models, WildReward-8B and ThinkPRM-7B, as well as confidence-based baselines, it largely underperforms our progress advantage that was not trained on this specific WebShop dataset. This implies a substantial challenge to building reliable PRMs on agentic tasks, where we offer a new angle on this.

\paragraph{Can progress advantage take benefit of task-specific RL training?} In the main body of the paper, our core statement was that the progress advantage automatically emerges after general post-training, and it is a sufficiently useful signal to guide or monitor agentic inference. Meanwhile, one may wonder if we could further make the signal sharper and tailored to a specific downstream task by further RL fine-tuning the behavior policy on a task of interest. In Table~\ref{tab:tts-training-based-rm}, we present the results on this by training \texttt{Qwen2.5-7B-Instruct} with GRPO~\cite{shao2024deepseekmath} on the WebShop dataset under the hyperparameter specification in Table~\ref{tab:grpohyp}. We see that progress advantage without any task-specific training already achieves the best performance, rivaling the ThinkPRM-14B, which doubles model size compared to the behavior or reference policy; Besides, constructing the progress advantage with $\log\pi_{\rm GRPO}(a|s)-\log\pi_{\texttt{Qwen2.5-7B-Instruct}}(a|s)$ further pushes the success rate about 8.5\% by sharpening the reward signal to downstream tasks.
\begin{table}[h]
\centering
\small
\caption{\textbf{Hyperparameter configuration for GRPO training on WebShop}.}
\label{tab:grpohyp}
\begin{tabular}{lc}
\toprule
Hyperparameter & Value \\
\midrule
Max prompt length & 4096 \\
Max response length & 512 \\
Max environment steps & 10 \\
Learning rate & $2 \times 10^{-6}$ \\
Success reward & 10 \\
Failure reward & 0 \\
Invalid action penalty & -0.1 \\
Group size & 8 \\
Rollout temperature & 1.0 \\
Validation temperature & 0.4 \\
Mini-batch size & 64 \\
KL coefficient & 0.01 \\
\bottomrule
\end{tabular}
\end{table}

\paragraph{Analysis on bootstrap confidence intervals.}
\label{sec:apdx:bootstrap}

\begin{table}[ht]
\caption{\textbf{Bootstrap confidence intervals for $\tau^2$-Airline UQ.} We report AUROC with 95\% percentile bootstrap intervals, computed from 10{,}000 resamples of trajectories with replacement.}
\label{tab:apdx:uq-bootstrap}
\centering
\small
\resizebox{\textwidth}{!}{%
\begin{tabular}{l|cccc|c}
\toprule
Scoring Method & \texttt{Gemma4-4B} & \texttt{Qwen3.5-9B} & \texttt{Qwen3-14B} & \texttt{Olmo3-7B} & Mean CI width \\ \midrule
\rowcolor{mygray}{\claudelogo} Sonnet-4.6 & 0.615 {\scriptsize[0.44, 0.78]} & 0.726 {\scriptsize[0.56, 0.87]} & 0.519 {\scriptsize[0.28, 0.73]} & 0.715 {\scriptsize[0.52, 0.90]} & 0.367 \\
WildReward-8B & 0.312 {\scriptsize[0.15, 0.49]} & 0.540 {\scriptsize[0.38, 0.70]} & 0.314 {\scriptsize[0.04, 0.60]} & 0.514 {\scriptsize[0.23, 0.79]} & 0.445 \\
ThinkPRM-7B   & 0.478 {\scriptsize[0.31, 0.65]} & 0.582 {\scriptsize[0.41, 0.74]} & 0.276 {\scriptsize[0.07, 0.55]} & 0.492 {\scriptsize[0.24, 0.75]} & 0.414 \\
ThinkPRM-14B  & 0.426 {\scriptsize[0.26, 0.61]} & 0.655 {\scriptsize[0.49, 0.80]} & 0.292 {\scriptsize[0.09, 0.53]} & 0.708 {\scriptsize[0.46, 0.92]} & 0.391 \\ \midrule
\rowcolor{lightpurple}
Progress Advantage & \textbf{0.860} {\scriptsize[0.72, 0.97]} & \textbf{0.735} {\scriptsize[0.58, 0.87]} & \textbf{0.727} {\scriptsize[0.44, 0.96]} & \textbf{0.757} {\scriptsize[0.54, 0.93]} & \textbf{0.360} \\ \bottomrule
\end{tabular}}
\end{table}

Many long-horizon agentic tasks provide a few samples; for example, $\tau^2$-Airline contains only 50 tasks, and some backbones solve only a few of them. To check how reliable the Airline UQ results are, we compute bootstrap confidence intervals in Table~\ref{tab:apdx:uq-bootstrap}. While we do not claim a statistically significant improvement in general, Progress Advantage achieves the highest point estimate for all four backbones, and its mean interval width is the smallest among all methods.

\section{Derivation of Implicit Rewards Under Stochastic MDP}\label{sec:apdx:derivation-reward}
The KL-constrained reward maximization problem is formulated as follows,
\begin{align} \label{eq:apdx:kl-constrained-rl-obj}
\max_{\pi_{\theta}} J(\pi_{\theta})
=\max_{\pi_{\theta}}\mathbb{E}_{a_{t}\sim\pi_{\theta}(\cdot|s_t)}\big[\sum_t r(s_t,a_t)
-\beta  \log \frac{\pi_{\theta}(a_t|s_t)}{\pi_{\rm ref}(a_t|s_t)} \big| s_0 \sim \rho \big],
\end{align}
\noindent where $\beta>0$ is the regularization coefficient and $\pi_{\rm ref}$ denotes a reference policy, commonly built with pre-trained or SFT checkpoints. This objective can be equally expressed as the following maximum entropy RL form~\cite{ziebart2008maximum,haarnoja17a} with entropy $H(\cdot)$,
\begin{align} \label{eq:max-ent-rl-obj}
\max_{\pi_{\theta}} J(\pi_{\theta})
=\max_{\pi_{\theta}}\mathbb{E}_{a_{t}\sim\pi_{\theta}(\cdot|s_t)}\big[\sum_t r(s_t,a_t)
+\beta  \log {\pi_{\rm ref}(a_t|s_t)} + \beta H\big(\pi_{\theta}(\cdot|s_t)\big) \big| s_0 \sim \rho \big].
\end{align}
This optimization problem gives us a known solution~\cite{ziebart2010modeling,rafailov2024from}, $\pi^{*}(a_t|s_t)=\exp\big( \frac{Q^{*}(s_t,a_t)-V^{*}(s_t)}{\beta} \big)$, corresponding optimal value function $V^{*}(s_t)=\beta\log\sum_{a}\exp\big( Q^{*}(s_t,a)/\beta \big)$, and also the corresponding optimal action-value function given the Bellman optimality equation,
{
\begin{align} \label{eq:apdx:bellman}
Q^{*}(s_t,a_t)=\begin{cases}
r(s_t,a_t)+\beta \log\pi_{\rm ref}(a_t|s_t)+\mathbb{E}_{s_{t+1}\sim f}[V^{*}(s_{t+1})], & \text{if}~s_{t+1}~\text{is not terminal} \\
r(s_t,a_t)+\beta \log\pi_{\rm ref}(a_t|s_t), & \text{otherwise.}
\end{cases}
\end{align}
}

Now we re-express the Eq~\ref{eq:apdx:bellman} as a reward-centric form and sum it across the trajectory up to $T-1$ position as below,
{\small
\begin{align} \label{eq:apdx:reward-derivation}
    \sum_{t=0}^{T-1}r(s_t,a_t)&=\sum_{t=0}^{T-1} \big( Q^{*}(s_t,a_t)-\beta\log\pi_{\rm ref}(a_t|s_t)-\mathbb{E}_{s_{t+1}}[V^{*}(s_{t+1})] \big) \nonumber\\
    &=\sum_{t=0}^{T-1} \big( \big[ \beta\log\pi^{*}(a_t|s_t)+V^{*}(s_t) \big]- \beta\log\pi_{\rm ref}(a_t|s_t)-\mathbb{E}_{s_{t+1}}[V^{*}(s_{t+1})] \big) \nonumber\\
    &=\sum_{t=0}^{T-1} \big( \beta\log\frac{\pi^{*}(a_t|s_t)}{\pi_{\rm ref}(a_t|s_t)} +V^{*}(s_t)-\mathbb{E}_{s_{t+1}}[V^{*}(s_{t+1})]\big).
\end{align}}
Eq.~\ref{eq:apdx:reward-derivation} indicates that we can no longer cancel out the intermediate value terms through the telescoping sum under stochastic MDP, and thus can not represent the exact reward solely with the known policy terms. This motivates us to explore an alternative derivation, progress advantage in Proposition~\ref{prop:prog_adv}.

\section{Missing Proof}\label{sec:apdx:derivation}
\subsection{Derivation of Progress Advantage}\label{sec:apdx:derivation-advantage}

In this section, we provide a full derivation from the optimal policy, the optimal state value function, the optimal action value function, and finally, the optimal advantage as an implicit process reward. 

\begin{proposition}[Restatement of Proposition~\ref{prop:prog_adv}] \label{prop:apdx:prog_adv}
    Let $\tilde\pi^{*}$ be an optimal policy under the KL-regularized RL objective (Eq.~\ref{eq:kl-constrained-rl-obj}) with $\beta>0$, shaped with the reference policy $\pi_{\rm ref}$ where $\pi_{\rm ref}(a|s)>0$ for any $a\in\mathcal{A}$ and $s\in\mathcal{S}$. 
    Then, the optimal advantage function is exactly recovered by the log probability ratio between $\tilde{\pi}^{*}$ and $\pi_{\rm ref}$ for any state and action:
    \begin{equation} \label{eq:apdx:prog_adv}
        \tilde{A}^{*}(s,a)=\tilde{Q}^{*}(s,a)-\tilde{V}^{*}(s)=\beta\log\frac{\tilde{\pi}^{*}(a|s)}{\pi_{\rm ref}(a|s)},~~~~\forall s \in \mathcal{S}, a \in \mathcal{A}.
    \end{equation}
\end{proposition}

\begin{proof}
We start from the KL-regularized RL objective given $\pi_{\rm ref}$ and $\rho$ without a discounting factor,
\begin{align} 
\max_{\pi_{\theta}} J(\pi_{\theta})
=\max_{\pi_{\theta}}\mathbb{E}_{a_t\sim\pi_{\theta},s_t\sim f}\big[\sum_{t=0}^{\infty} r(s_t,a_t)
-\beta  \log \frac{\pi_{\theta}(a_t|s_t)}{\pi_{\rm ref}(a_t|s_t)} \big| s_0 \sim \rho \big],
\end{align}
\noindent where we construct the infinite sum of any finite $T$-length sequence by having an absorbing state with reward zero. This reward maximization problem can be also expressed as expected value maximization, $\max_{\pi_{\theta}}J(\pi_{\theta})=\max_{\pi_{\theta}}\mathbb{E}_{s_{0}\sim\rho}[\tilde{V}^{\pi_{\theta}}(s_{0})]$, by definition of the state value function,
\begin{align} \label{eq:apdx:kl-constrained-value}
\tilde{V}^{\pi_{\theta}}(s)
=\mathbb{E}_{a_{t}\sim\pi_{\theta}(\cdot|s_t)}\big[\sum_{t=0}^{\infty} r(s_t,a_t)
-\beta  \log \frac{\pi_{\theta}(a_t|s_t)}{\pi_{\rm ref}(a_t|s_t)} \big| s_0 = s\big].
\end{align}
Given the time-homogeneous transition probability $f(s'|s,a)$, by unrolling one step of the trajectory, we have
\begin{align} \label{eq:apdx:kl-constrained-value-bellman}
\tilde{V}^{\pi_{\theta}}(s)
=\mathbb{E}_{a\sim\pi_{\theta}(\cdot|s)}\big[r(s,a)
-\beta  \log \frac{\pi_{\theta}(a|s)}{\pi_{\rm ref}(a|s)}+\mathbb{E}_{s'\sim f(\cdot|s,a)}[\tilde{V}^{\pi_{\theta}}(s')] \big].
\end{align}
\noindent by renaming $a_0$ as $a$ and $s_1$ as $s'$. It satisfies its own (soft) Bellman equation augmented by the KL penalty. Now, we have the following action value function by definition:
\begin{align} \label{eq:apdx:kl-constrained-q}
\tilde{Q}^{\pi_{\theta}}(s,a)
=r(s,a)+\mathbb{E}_{s'\sim f(\cdot|s,a)}[\tilde{V}^{\pi_{\theta}}(s')].
\end{align}
Then, plugging Eq.~\ref{eq:apdx:kl-constrained-q} into Eq.~\ref{eq:apdx:kl-constrained-value-bellman}, we get the following for any $\pi_{\theta}$ and $s$.
\begin{align} \label{eq:apdx:kl-constrained-value-q}
\tilde{V}^{\pi_{\theta}}(s)
=\mathbb{E}_{a\sim\pi_{\theta}(\cdot|s)}\big[\tilde{Q}^{\pi_{\theta}}(s,a)
-\beta  \log \frac{\pi_{\theta}(a|s)}{\pi_{\rm ref}(a|s)} \big].
\end{align}

Since the above equation holds for any $\pi_{\theta}$ and $s$, we can now re-express our objective at state $s$ given a fixed optimal action value function $Q^{*}(s,a)$ as follow,
\begin{align} \label{eq:apdx:kl-constrained-alt-obj}
\max_{\pi_{\theta}} \tilde{J}(\pi_{\theta}):=\max_{\pi_{\theta}}\sum_{a}\pi_{\theta}(a|s)\big[\tilde{Q}^{*}(s,a)
-\beta  \log \frac{\pi_{\theta}(a|s)}{\pi_{\rm ref}(a|s)} \big],~~~\text{s.t.}~~~\sum_{a}\pi_{\theta}(a|s)=1.
\end{align}
The optimal policy for this local step-level objective is equivalent to that of the global trajectory-level objective by the Policy Improvement Theorem~\cite{howard1960dynamic,sutton2018reinforcement}. We then solve this constrained optimization problem with the method of Lagrangian multipliers to get the optimal policy,
\begin{align} \label{eq:apdx:kl-constrained-alt-obj-solve}
\tilde{J}(\pi_{\theta},\lambda)&=\sum_{a}\pi_{\theta}(a|s)\big[\tilde{Q}^{*}(s,a)
-\beta  \log \frac{\pi_{\theta}(a|s)}{\pi_{\rm ref}(a|s)} \big]+\lambda\big(1-\sum_{a}\pi_{\theta}(a|s)\big)\\
\frac{\delta\tilde{J}(\pi_{\theta},\lambda)}{\delta\pi_{\theta}(a|s)}&=\tilde{Q}^{*}(s,a)-\beta(\log\frac{\pi_{\theta}(a|s)}{\pi_{\rm ref}(a|s)}+1)-\lambda \equiv 0\\
\tilde{\pi}^{*}(a|s)&=\pi_{\rm ref}(a|s)\exp(\frac{1}{\beta}\tilde{Q}^{*}(s,a))\exp(-\frac{\lambda}{\beta}-1)\\
&=\frac{1}{Z(s)}\pi_{\rm ref}(a|s)\exp(\frac{1}{\beta}\tilde{Q}^{*}(s,a)),
\end{align}
\noindent where $Z(s)=\sum_{a}\pi_{\rm ref}(a|s)\exp(\frac{1}{\beta}\tilde{Q}^{*}(s,a))$. Applying log-linearization induces the following,
\begin{align} \label{eq:apdx:kl-constrained-alt-obj-solve-lin}
    \beta\log\frac{\tilde{\pi}^{*}(a|s)}{\pi_{\rm ref}(a|s)}=\tilde{Q}^{*}(s,a)-\beta\log Z(s)
\end{align}
Plugging this Eq.~\ref{eq:apdx:kl-constrained-alt-obj-solve-lin} into Eq.~\ref{eq:apdx:kl-constrained-value-q} induces the following optimal state value function,
\begin{align} \label{eq:apdx:kl-constrained-alt-value-solve}
    \tilde{V}^{*}(s)=\sum_{a}\tilde{\pi}^{*}(a|s)[\beta\log Z(s)]=\beta \log Z(s) \sum_{a}\tilde{\pi}^{*}(a|s) = \beta \log Z(s),
\end{align}
\noindent where $\tilde{Q}^{*}(s,a)$ was canceled. With Eq.~\ref{eq:apdx:kl-constrained-alt-obj-solve-lin} and Eq.~\ref{eq:apdx:kl-constrained-alt-value-solve}, we finally get our optimal advantage function,
\begin{align} \label{eq:apdx:kl-constrained-alt-advantage}
    \beta\log\frac{\tilde{\pi}^{*}(a|s)}{\pi_{\rm ref}(a|s)}=\tilde{Q}^{*}(s,a)-\tilde{V}^{*}(s)=\tilde{A}^{*}(s,a),
\end{align}
\noindent defined solely by the log-probability ratio. The stochasticity of the state transition $s'\sim f(\cdot|s,a)$ is embedded in $Q^*$ and $V^{*}$ under this general non-deterministic MDP. If we want to model one of them separately or want to directly model the reward function, we have to explicitly deal with that stochasticity. By embracing the optimal advantage function as a pseudo reward, we bypass the explicit stochasticity modeling while leaving its implicit reflection to the log-probability term computed solely from the realized observation.
\end{proof}

\subsection{Proof: Clipping Surrogate RL as an Implicit KL Constraint}\label{sec:apdx:clip_kl_proof}
\begin{proposition}[Restatement of Proposition~\ref{proposition:clip_kl_solution}]
\label{proposition:apdx:clip_kl_solution}
Let $\pi_{\rm ref}$ and $\pi_\theta$ be the reference and target policies sharing the same support. Define the importance sampling ratio as $R(s,a) = \frac{\pi_\theta(a \mid s)}{\pi_{\rm ref}(a \mid s)}$. If optimization enforces a per-sample constraint $R(s, a) \in [1 - \varepsilon, 1 + \varepsilon]$ for all $(s, a)$ and a small $\varepsilon > 0$, then $D_{\rm KL}(\pi_{\theta} \parallel \pi_{\rm ref}) \leq \frac{\varepsilon^2}{2}$ and $D_{\rm KL}(\pi_{\rm ref} \parallel \pi_{\theta}) \leq \frac{\varepsilon^2}{2}$, similarly for reverse KL, locally at $R(s,a) \approx 1$. 
\end{proposition}

\begin{proof}
Let $\delta(s,a) = R(s,a) - 1$. Then, the clipping surrogate RL optimization problem enforces $|\delta(s,a)| \leq \varepsilon$. We will show that any policy found under PPO-Clip style surrogate optimization is strictly within a KL (both reverse and forward) trust region of radius $\epsilon^{2}/2$.

First, the reverse KL divergence is defined as $D_{\rm KL}(\pi_{\rm ref} \parallel \pi_\theta) = \mathbb{E}_{\pi_{\rm ref}}[-\log R]$. The second-order Taylor expansion of $-\log R=-\log(1 + \delta)$ around $R = 1$ yields:
\begin{equation}
    -\log(1 + \delta) = -\delta + \frac{\delta^2}{2} + \mathcal{O}(\delta^3) \label{eq:kl_rev_taylor}
\end{equation}

Similarly, the forward KL divergence is defined as $D_{\rm KL}(\pi_\theta \parallel \pi_{\rm ref}) = \mathbb{E}_{\pi_\theta}[\log R]$. With an importance sampling ratio $R$, it can be equivalently expressed as $\mathbb{E}_{\pi_{\rm ref}}[R \log R]$. The second-order Taylor expansion of $R \log R = (1+\delta)\log(1+\delta)$ around $R = 1$ yields:
\begin{equation}
    (1 + \delta)\log(1 + \delta) \approx (1 + \delta)\left(\delta - \frac{\delta^2}{2}\right) = \delta + \frac{\delta^2}{2} + \mathcal{O}(\delta^3) \label{eq:kl_for_taylor}
\end{equation}

Now, taking the expectation over $\pi_{\rm ref}$ for both expansions, Eq.~\ref{eq:kl_rev_taylor} and Eq.~\ref{eq:kl_for_taylor}, we have:
\begin{align}
    D_{\rm KL}(\pi_{\rm ref} \parallel \pi_\theta) \approx \mathbb{E}_{\pi_{\rm ref}}\left[-\delta + \frac{\delta^2}{2}\right],~~~~~ D_{\rm KL}(\pi_\theta \parallel \pi_{\rm ref}) \approx \mathbb{E}_{\pi_{\rm ref}}\left[\delta + \frac{\delta^2}{2}\right]
\end{align}
Since $\pi_{\rm ref}$ and $\pi_\theta$ share the support and are valid probability distributions, $\mathbb{E}_{\pi_{\rm ref}}[R] = \mathbb{E}_{\pi_{\rm ref}}[\frac{\pi_{\theta}}{\pi_{\rm ref}}] = 1$. That means, $\mathbb{E}_{\pi_{\rm ref}}[\delta] = 0$, i.e., the exact removal of the linear term, collapsing both KL divergences to the scaled Pearson $\chi^2$-divergence:
\begin{equation}
    D_{\rm KL} \approx \frac{1}{2}\mathbb{E}_{\pi_{\rm ref}}[\delta^2] \label{eq:kl_taylor}
\end{equation}
Given the per-sample constraint $|\delta(s,a)| \leq \varepsilon$, we have $\mathbb{E}_{\pi_{\rm ref}}[\delta^2] \leq \varepsilon^2$. Then, substituting this into our approximation, Eq.~\ref{eq:kl_taylor}, yields below:
\begin{equation}
    D_{\rm KL} \lesssim \frac{\varepsilon^2}{2} \label{eq:kl_eps_bound}
\end{equation}
for both the forward and reverse directions. Therefore, any feasible policy under the PPO-Clip style constrained RL becomes a proper subset of a KL trust region of radius $\delta = \frac{\varepsilon^2}{2}$, regardless of whether the objective applies a forward or reverse KL penalty.
\end{proof}

\section{Broader Context and Discussion} \label{sec:apdx:context}

\subsection{Outcome Reward Modeling} 
The early driving force in preference learning and reinforcement learning of LLMs was the outcome reward models (ORMs)~\cite{cobbe2021training,uesato2022solving,yu2024ovm}. They provide supervision at the level of final responses, rewarding reasoning trajectories according to the quality of their ultimate outcomes, regardless of their process. The annotations over outcomes are relatively easy to collect, making ORMs a scalable choice for reranking, search, and RL for reasoning models~\cite{bai2022training,shao2024deepseekmath}. Nevertheless, ORM supervision is coarse and delayed~\citep{lightman2023let,lyu2025exploring}: it evaluates only the final outcome and does not verify whether the intermediate reasoning process is sound. Consequently, trajectories that contain errors or spurious steps may still be rewarded if they happen to produce the correct answer. These limitations motivate process reward models (PRMs), which instead provide fine-grained supervision over intermediate reasoning steps. We propose a practical approach to building PRMs on agentic scaffolding to utilize the fine-grained signals in broad inference-time applications.

\subsection{Implicit Reward Modeling} 
After seminar works published~\cite{rafailov2023direct,NEURIPS2023_3be7859b,rafailov2024from}, implicit reward formulation has gained its bold popularity due to its practical merit, allowing users to bypass explicit reward modeling. The implicit reward-based approach has usually been developed in a preference fine-tuning setup~\cite{meng2024simpo,hong2024orpo,azar2024general,ethayarajh2024kto,wang2024beyond} but also extended to general multi-step reasoning or multi-step interaction setups recently~\cite{lai2024step,shi2024direct,yuan2025free,cui2025process,liu2025agentic}. However, its remain opaque to adopt this implicit reward modeling technique on the LLM agent setup where the rollouts of the agents are made up with not only the model's own deterministic token completion but also with stochastic observations from the environmental entities. To fill this gap, we establish a foundation of implicit reward under stochastic MDP tied with the realistic LLM agents inference settings, by deriving the progress advantage formulation.

\subsection{Distribution Contrasting and Sharpening}
In essence, progress advantage contrasts the likelihood of a behavior policy with that of a reference policy. This kind of contrastive probabilistic quantity has a rich history in statistics and the machine learning field. For instance, in its vanilla form, one can easily draw an interpretation as a likelihood ratio test statistic~\cite{woolf1957the,vuong1989likelihood} as well as a logistic regression model in the noise contrastive estimation~\cite{gutmann2010noise,gutmann2012noise,collins2018noise}. If we take the expectation over it, given a distribution, it becomes a measure of contrastive divergence~\cite{hinton2002training,carreira2005contrastive,oh2022learning} with the KL divergence instantiation. In addition to this luxuriant connection with the classic, progress advantage has also dense relationships with contemporary findings in LLM research, e.g., \textit{contrastive decoding}~\cite{li2023contrastive,o2023contrastive} that builds a keen output token probability distribution by contrasting the probabilities of an expert and an amateur model, as well as an interpretation of RL post-training as a \textit{distribution sharpening}~\cite{he2025rewarding,karan2026reasoning,oh2026sharpening}. Despite this wide spectrum of relevant literature, it still remains untapped to use this contrastive likelihood quantity as an implicit reward signal under agentic harness. We expect our progress advantage and its connection to the broader concepts we discussed here to present insights and exciting follow-ups.

\subsection{Self-improving Intelligent Systems}
Huang et al. \citep{huang2025self} provide a novel view, interpreting LLM self-improvement methods as a type of sharpening. That is, an LLM-based system leveraging its own log-likelihood as a self-reward signal, improving itself without external supervision \textit{by concentrating probability mass on high-quality generations}. In a similar spirit, progress advantage extracts a self-contained progress signal from the log-ratio between an RL-trained policy (future-self) and its reference policy (past-self), requiring no additional reward model or process annotation. While we validate this signal in a single-cycle test-time trajectory selection, monitoring, and failure attribution, it can be viewed as one stage in a longer self-improvement lifecycle: agents generate trajectories, internally score their progress, and eventually distill such signals back into future policies~\cite{bai2022constitutional,zelikman2022star,yuan2024selfrewarding,wu2025meta}.

\subsection{Open-source AI and Sustainable Machine Learning}
By envisioning the promise of updatable and sustainable machine learning~\cite{raffel2023building}, some researchers have explored how artifacts produced during the model development process, e.g., model weight checkpoints, can be shared across community members and reused efficiently~\cite{wortsman22a,NEURIPS2022_bc6cddcd,ilharco2023editing}. This line of work has renewed interest in the classical model ensemble~\cite{hoeting1999bayesian,492f6c68703a4b6d97bb8509d817d00f,gal16dropout,NIPS2017_9ef2ed4b,huang2017snapshot} and has recently led to substantial progress under the term \textit{model merging}~\cite{yang2026model,yadav2023ties,yu2024language,akiba2025evolutionary}. 
Another line of work is \textit{black-box prompt optimization}~\cite{sun2022black,deng2022rlprompt,cheng2024black,oh2026robust}, which solely leverages the output probabilities from the black-box models to customize the model's input space interface without knowledge about model architectures or parameters.
By making the development process and artifacts more transparent and broadly accessible, we can reduce the waste of repeatedly ``reinventing the wheel'', while creatively leveraging the existing common goods to advance our own intelligence systems. We hope that our progress advantage inspires such creative adoption of existing resources. Note that the pursuit of this kind of sustainable machine learning is not limited to the reuse of model weights or outputs. We may need to revisit the entire development pipeline~\cite{manchanda2024open,hall2025marin}, i.e., from the earliest stage of data collection, to find out whether \textbf{neglected free lunches} remain~\cite{han2023neglected}.

\section{Broader Impacts} \label{sec:apdx:impacts}
Progress advantage offers a way to score trajectories from agentic systems without undergoing the cost-heavy development phase. This can contribute to reducing the total GPU hours from groups that may want to adopt progress advantage during their LLM deployment under agentic harness, while assisting safety and outlier monitoring, as well as improving the system response quality. Meanwhile, since the progress advantage produces reward signals based on the log probability ratio of trained LLM policies, it may contain historical and social bias hidden in the pre/mid/post-train datasets, resulting in biased scoring and preference during its deployment-time applications. Therefore, one should take care of the trajectories gone through the progress advantage monitoring before the public release by leveraging some well-established safeguards tools~\cite{inan2023llama,xiang2024guardagent,wang2024self}.

\section{Computing Resource Statement} \label{sec:apdx:compute}
We used NVIDIA A6000, A100-SXM4, and H200 in a mix, but reported the GPU hours in terms of a single A100. To reproduce our progress advantage (as well as component-wise ablation) in best-of-N experiments of the TTS scenario, the full sweep may require 26 hours; UQ may require 16 hours, and the FA will not take much, up to 4 hours, resulting in 46 hours in total for only our method. Entire replication, including ours and the full baseline, and the $k$-smoothed variant of progress advantage, would take four times longer. Besides, the rough per-model (and per-pair) memory requirement dominated by the weight footprint for progress advantage is described in Table~\ref{tab:compute_pair_memory}.
\begin{table}[h]
\caption{\textbf{bf16 weight footprint per backbone pair (policy and reference).}}
\label{tab:compute_pair_memory}
\centering
\small
\begin{tabular}{lcc}
\toprule
Backbone & Per-model & Pair peak \\
\midrule
\texttt{Qwen3.5-9B} & $\sim$18\,GB & $\sim$36\,GB \\
\texttt{Qwen3-14B} & $\sim$28\,GB & $\sim$56\,GB \\
\texttt{Qwen2.5-7B}  & $\sim$14\,GB & $\sim$28\,GB \\
\texttt{Gemma4-4B} & $\sim$8\,GB  & $\sim$16\,GB \\
\texttt{Olmo3-7B} & $\sim$14\,GB & $\sim$28\,GB \\
\bottomrule
\end{tabular}
\end{table}

% \clearpage
\section{Prompt Template for Baseline Methods} \label{sec:apdx:prompts}
\begin{figure}[htb]
\centering
\begin{minipage}{0.95\columnwidth}
\small
\begin{promptbox}
You are given a task and a proposed step-by-step solution:

[Task] \n {task}

[Solution] \n {solution}

Review and critique each step in the proposed solution to determine
whether each step is correct. If the solution is incomplete, only
verify the provided steps. For each step, end your critique with a
single line of the form `Step N is {correct}' or
`Step N is {incorrect}'.
\end{promptbox}
\end{minipage}
\caption{Prompt template used for \texttt{ThinkPRM} in TTS and UQ settings. \texttt{\{task\}} is the initial user query; \texttt{\{solution\}} is the agent's task-solving trajectory interacting with tools and/or the user.}
\label{fig:thinkprm-uqtts-prompt}
\end{figure}

\begin{figure}[htb]
\centering
\begin{minipage}{0.95\columnwidth}
\small
\begin{promptbox}
You are given a multi-agent problem-solving trajectory 
and must verify each step.

[Task] \n {task}

[Solution] \n {solution}

Review and critique each step in the proposed solution to determine
whether each step is correct. For each step, end your critique with
a single line of the form `Step N is {correct}' or
`Step N is {incorrect}'.
\end{promptbox}
\end{minipage}
\caption{Prompt template used for \texttt{ThinkPRM} in the FA setting.}
\label{fig:thinkprm-fa-prompt}
\end{figure}

\begin{figure}[htb!]
\centering
\begin{minipage}{0.9\columnwidth}
\small
\begin{promptbox}
You are a binary success predictor for LLM agent trajectories. Given the agent-user interaction trajectory, predict whether the agent succeeds.

Trajectory: \n {trajectory}

Output success = 1 if the agent completes the user's goal correctly.
Output success = 0 if the agent fails, leaves the task incomplete,
violates policy, makes a material factual/tool-use error, or causes
a wrong outcome.

When uncertain, choose the more likely label and report confidence.
Return JSON only:
{
  "success": <0 or 1>,
  "confidence": <number between 0.0 and 1.0>,
  "justification": "<one or two sentences explaining the prediction>"
}
\end{promptbox}
\end{minipage}
\caption{Prompt used for the LLM-as-a-Judge (Claude-Sonnet-4.6~\cite{anthropic2026sonnet46}) baseline in UQ experiment.}
\label{fig:uq-judge-prompt}
\end{figure}

\begin{figure}[htb]
\centering
\begin{minipage}{0.9\columnwidth}
\small
\begin{promptbox}
# Task Description
You are an expert conversation evaluator. Your task is to judge the
**User's Satisfaction** with the Assistant's response based on the
conversation context.
Please rate the response on a scale of 1 to 5 integers.

# Scoring Criteria
[1] CLEARLY NEGATIVE / REJECTION
[2] CORRECTION / ERROR POINTER (Negative)
[3] NEUTRAL
[4] POSITIVE ENGAGEMENT
[5] CLEAR SATISFACTION

# Input Data
## Context (History) \n {history}
## User Query \n {query}
## Assistant Response \n {response}

# Output
Based on the criteria above, please output ONLY the integer score 
(1, 2, 3, 4, or 5).
\end{promptbox}
\end{minipage}
\caption{Prompt template used for the outcome reward model baseline \texttt{WildReward-8B} across all downstream applications. Note that the trajectory-level reward is defined as $1+\sum_{j=1}^{4}\sigma(z_j)$, the sum of four ordinal regression threshold logits $z_j$ obtained directly from the model's classification head; we did not actually generate the integer that the prompt requests.}
\label{fig:wildreward-prompt}
\end{figure}
% \newpage
% \input{arxiv2605/checklist}
%%%%%%%%%%%%%%%%%%%%%%%%%%%%%%%%%%%%%%%%%%%%%%%%%%%%%%%%%%%%

\end{document}